\documentclass[11pt]{article}
\usepackage{amsmath,amssymb,amsthm}
\usepackage{latexsym}
\usepackage{graphicx}
\usepackage{algorithm}
\usepackage{algorithmic}
\usepackage{authblk}

\newtheorem{theo}{Theorem}
\newtheorem{lem}{Lemma}
\newtheorem{df}{Definition}
\newtheorem{cor}{Corollary}

\def\R{{\mathbb{R}}}

\def\N{{\mathbb{N}}}

\usepackage{ascmac}
\usepackage{bm}
\def\a{{\bm a}}
\def\d{{\bm d}}
\def\e{{\bm e}}
\def\x{{\bm x}}
\def\y{{\bm y}}
\def\u{{\bm u}}
\def\0{{\bm 0}}
\def\1{{\bm 1}}
\def\Rbb{\mathbb{R}}
\def\Ebb{\mathbb{E}}

\title{Parallel Distributed Block Coordinate Descent Methods based on Pairwise Comparison Oracle}

\author[1]{Kota Matsui}
\author[2]{Wataru Kumagai}
\author[3]{Takafumi Kanamori}

\affil[1]{Nagoya Institute of Technology}
\affil[2]{Kanagawa University}
\affil[3]{Nagoya University}
\date{}

\begin{document}
\maketitle

\begin{abstract}
This paper provides a block coordinate descent algorithm to solve unconstrained optimization problems. In our algorithm, computation of function values or gradients is not required. Instead, pairwise comparison of function values is used. Our algorithm consists of two steps; one is the direction estimate step and the other is the search step. Both steps require only pairwise comparison of function values, which tells us only the order of function values over two points. In the direction estimate step, a Newton type search direction is estimated. A computation method like block coordinate descent methods is used with the pairwise comparison. In the search step, a numerical solution is updated along the estimated direction. The computation in the direction estimate step can be easily parallelized, and thus, the algorithm works efficiently to find the minimizer of the objective function. Also, we show an upper bound of the convergence rate. In numerical experiments, we show that our method efficiently finds the optimal solution compared to some existing methods based on the pairwise comparison. 
\end{abstract}

\section{Introduction}
\label{int}
Recently, demand for large-scale complex optimization is increasing in computational science, engineering
and many of other fields. In that kind of problems, there are many difficulties caused by  
noise in function evaluation, many tuning parameters and  high computation cost.  
In such cases, derivatives of the objective function 
are unavailable or computationally infeasible. These problems can be treated by the derivative-free optimization (DFO) 
methods.  

DFO is the tool for optimization without derivative information of the objective function
and constraints, and it has been widely studied for decades~\cite{conn2009introduction,rios2013derivative}. 
%For example, 
DFO algorithms include gradient descent methods with a finite difference gradient
estimation~\cite{Fu06:sto_grad_est,Flaxman:2005:OCO:1070432.1070486},  
some direct search methods using only function
values~\cite{audet2002analysis,nelder65:_simpl_method_funct_minim},  
and trust-region methods~\cite{conn2000trust}. 

There is, however, a more restricted setting in which not only derivatives
but also values of the objective function are unavailable or computationally infeasible.
In such a situation, the so-called pairwise comparison oracle, that tells us an order of
function values on two evaluation points, is used instead of derivatives and function
evaluation~\cite{nelder65:_simpl_method_funct_minim,jamieson12:_query_compl_deriv_free_optim}. 
For example, the pairwise comparison is used in learning to rank to collect training samples to estimate 
the preference function of the ranking problems~\cite{mohri2012foundations}. In decision
making, 
finding the most preferred feasible solution from among the set of many alternatives is an 
important application of ranking methods using the pairwise comparison. 
%In addition, %another oracle type information 
Also, other type of information such as stochastic gradient-sign oracle has been
studied~\cite{ramdas2013algorithmic}. 

Now, let us introduce two DFO methods, i.e., the Nelder-Mead
method~\cite{nelder65:_simpl_method_funct_minim} and stochastic coordinate descent
algorithm~\cite{jamieson12:_query_compl_deriv_free_optim}. They are closely related to our
work. In both methods, the pairwise comparison of function values is used as a building
block in optimization algorithms. 

%Firstly,
Nelder and Mead's downhill simplex method~\cite{nelder65:_simpl_method_funct_minim} was
proposed in early study of algorithms based on the pairwise comparison of function
values. 
In each iteration 
of the algorithm, a simplex that approximates the objective function is constructed
 according to ranking of function values on sampled points. 
Then, the simplex receives four operations, namely,  reflection, expansion, contraction and reduction  
in order to get close to the optimal solution. 
Unfortunately, the convergence of the Nelder-Mead algorithm is theoretically guaranteed
only in low-dimension problems~\cite{Lagarias98convergenceproperties}. %Furthermore, 
In high dimensional problems, the Nelder-Mead algorithm works poorly as shown in
\cite{gao12:_implem_nelder_mead_simpl_algor_adapt_param}. 

%It should be noted here that ranking of function values in the Nelder-Mead 
%algorithm~\cite{nelder65:_simpl_method_funct_minim} can be realized by a combination of 
%deterministic PC oracle.  

%Secondly, 
The stochastic coordinate descent algorithm using only the noisy pairwise comparison was
proposed in \cite{jamieson12:_query_compl_deriv_free_optim}. 
Lower and upper bounds of the convergence rate were also presented in terms of the number
of pairwise comparison of function values, i.e., query complexity. 
The algorithm iteratively solves one dimensional optimization problems like the coordinate
descent method. However, practical performance of the optimization algorithm was not
studied in that work. 

%Jamieson et al. propose 
%Lower and upper bounds of the convergence rate of their algorithm. 
%Moreover, they analyze the query complexity, i.e. the number of pairwise comparison which
%is required by the algorithm. 

%Jamieson et al.'s coordinate descent 
%The algorithm can apply to large-scale optimization problems because 
%in which the objective function is minimized with respect to one of the coordinates while the other
%coordinates are held fixed.   
% not shown. 

%As another oracle type information in optimization problems, 
%Ramdas {\it et al.}~\cite{ramdas2013algorithmic} consider a stochastic gradient-sign oracle, a noisy sign
%of components of a gradient. 
%They propose a stochastic coordinate descent algorithm based on gradient sign oracle and active learning 
%as a line search. Function error of this algorithm converge with exponential rate. 

In this paper, we focus on optimization algorithms using the pairwise comparison oracle. %and propose an efficient DFO algorithm. 
In our algorithm, the convergence to the optimal solution is guaranteed, when the number
of pairwise comparison tends to infinity. 
Our algorithm is regarded as a block coordinate descent method consisting of two steps:
the direction estimate step and search step. 
In the direction estimate step, the search direction is determined. 
% by the line search based on the pairwise comparison oracle.  
%is used along some coordinates to determine the search direction.  %Then 
In the search step, the current solution is updated along the search direction with an
appropriate step length. 
%Since %line searches in 
In our algorithm, the direction estimate step is easily parallelized. %with respect to each coordinate, 
%we can efficiently compute a search direction.
Therefore, it is expected %prospective 
that our algorithm effectively works even in large-scale optimization problems. 

%Now
Let us summarize the contributions presented in this paper. %which are described in this paper. 
\begin{enumerate}
	\item We propose a block coordinate descent algorithm based on the pairwise
	      comparison oracle, and point out that the algorithm is easily parallelized. 
	\item We derive an upper bound of the convergence rate 
	      % for strongly convex and strongly smooth objective functions 
	      in terms of the number of pairwise
	      comparison of function values, i.e., query complexity. 
	      % when objective functions are strongly convex and strongly smooth. 
	      %	In particular case, 
	      %	this condition can be relaxed (it is mentioned in  section~\ref{pre}). 
              % \item We analyze query complexity of our algorithm under both deterministic and stochastic 
	      %      pairwise comparison oracle. 
	\item We show a practical efficiency of our algorithm through numerical experiments. 
\end{enumerate}

%This paper relates the several researches, for example, initiated by the following researchers. 
 
%Tseng~\cite{tseng2001convergence} analyze a cyclic block coordinate descent algorithm to an 
%non differentiable function. 
%He show that if the objective function is suitably separable,  
%the convergence properties are guaranteed even if the objective function is discontinuous. 

%Bradley {\it et al.}~\cite{bradley2011parallel} propose a parallel distributed stochastic coordinate 
%descent, which is called {\it shotgun}, to minimize a nonnegative convex loss function with $L^1$ 
%regularization term. They obtain the upper bound of the number of parallel updates in which {\it shotgun} 
%achieves linear speedups compared with non parallelized stochastic coordinate descent. 

The rest of the paper is organized as follows. In Section~\ref{pre}, we explain the problem setup and give 
some definitions. Section~\ref{mai} is devoted to the main results. 
The convergence properties and query complexity of our algorithm are shown in the section. 
In Section~\ref{num}, numerical examples are reported. 
%The experimental performance of our algorithm, under both deterministic and stochastic
%oracle, is presented. 
Finally in Section~\ref{con}, we conclude the paper with the discussion on future works. 
All proofs of theoretical results are found in appendix. 
%Appendix~\ref{pro}. 

%%%%%%%%%%%%%%%%%%%%%%%%%%%%%%%%%%%%%%%%%%%%%%%%%%%%%%%%%%%%
%%%%%%%%%%%%%%%%%%%%%%%%%%%%%%%%%%%%%%%%%%%%%%%%%%%%%%%%%%%%
\section{Preliminaries}
\label{pre}
In this section, we introduce the problem setup and prepare some definitions and notations used throughout the paper. 
A function $f : \mathbb{R}^n \rightarrow \mathbb{R}$ is said to be $\sigma$-strongly convex on $\mathbb{R}^n$ 
for a positive constant $\sigma$, if for all $\x, \y \in \mathbb{R}^n$, the inequality 
\[
f(\y) \geq f(\x) + \nabla f(\x)^{T} (\y-\x) + \frac{\sigma}{2} \|\x-\y\|^2
\label{eqP1}
\]
holds, where $\nabla f(\x)$ and $\|\cdot\|$ denote the gradient of $f$ at $\x$ and the euclidean norm, respectively. 
The function $f$ is $L$-strongly smooth for a positive constant $L$,  
if $\|\nabla f(\x) -\nabla f(\y)\| \le{}L\|\x -\y\|$ holds for all $\x,\y \in \mathbb{R}^n$. 
The gradient $\nabla{f(\x)}$ of the $L$-strongly smooth function $f$ is referred to as $L$-Lipschitz gradient. 
%In other words, $f$ has $L$-Lipschitz gradient 
%if for all $\x, \y \in \mathbb{R}^n$, the inequality 
%\[
%f(\y) \leq f(\x) + \nabla f(\x)^{T} (\y-\x) + \frac{L}{2} ||\x-\y||^2
%\label{eqP2}
%\]
%holds. 
The class of $\sigma$-strongly convex and $L$-strongly smooth functions on $\Rbb^n$ is denoted as 
$\mathcal{F}_{\sigma,L}(\R^n)$. 
In the convergence analysis, mainly we focus on the optimization of objective functions in
$\mathcal{F}_{\sigma,L}(\R^n)$. 

We consider the following pairwise comparison oracle defined in \cite{jamieson12:_query_compl_deriv_free_optim}. 
\begin{df}[Pairwise comparison oracle]
The stochastic pairwise comparison (PC) oracle is a binary valued  random variable
$O_{f} : \mathbb{R}^n \times \mathbb{R}^n \rightarrow \{\ -1, +1 \}$ defined as %which satisfies 
\begin{equation}
\mathrm{Pr}[O_{f}(\x,\y) = {\rm sign} \{f(\y) - f(\x) \}] \ge \frac{1}{2} 
+ \min \{\delta_0, \mu |f(\y) - f(\x)|^{\kappa-1} \}, 
%O_{f}(x, y) := {\rm sign} \{f(y) - f(x) \}
\label{eqP3}
\end{equation}
where $0 < \delta_0 \le 1/2$, $\mu > 0$ and $\kappa \ge 1$. 
For $\kappa = 1$, without loss of generality $\mu \le \delta_0 \le 1/2$ is assumed. 
%Especially, 
When the equality 
\begin{equation}
\label{eqP4}
\mathrm{Pr}[O_{f}(\x,\y) = {\rm sign} \{f(\y) - f(\x) \}] = 1
\end{equation}
is satisfied for all $\x$ and $\y$, we call $O_{f}$ the deterministic PC oracle.
\end{df}
For $\kappa = 1$, the probability in (\ref{eqP3}) is not affected by the difference
$|f(\y) - f(\x)|$, meaning that the probability for the output of the PC oracle is not
changed under any monotone transformation of $f$.   
%This implies that  when $\kappa = 1$ holds, 
%That is, for $\kappa = 1$, it is sufficient for optimization to assume that 
%a monotone transformation of $f$ is an element of ${\cal F}_{\sigma, L}(\R^n)$, even if $f$ itself is not. 

%Jamieson et al.~\cite{jamieson12:_query_compl_deriv_free_optim}
%not only PC oracle but also noisy function evaluation oracle
%$E_{f}(\x) = f(\x) + \epsilon$ was considered, where $\epsilon$ is a random variable with mean zero and finite variance.  
%It is obvious that function evaluation oracle can always be transformed into PC oracle by defining a PC oracle as 
%$O_{f}(\x,\y) = {\rm sign} \{E_{f}(\y) - E_{f}(\x) \}$. 
%Jamieson et al. derived lower and upper bounds of  convergence rate of an optimization algorithm 
In~\cite{jamieson12:_query_compl_deriv_free_optim}, 
Jamieson et al. derived lower and upper bounds of convergence rate of an optimization
algorithm using the stochastic PC oracle. The algorithm is referred to as 
the original PC algorithm in the present paper. 
%; see also Theorem~\ref{lower}. 

Under above preparations, our purpose is to find the minimizer $\x^*$ of the objective
function $f(\x)$ in $\mathcal{F}_{\sigma, L}(\R^n)$ by using PC oracle. In the following section, we provide a DFO algorithm 
in order to solve the optimization problem and consider the convergence properties including query complexity. 

%\begin{comment}
%This problem can divide into two sub-problems: (i) Construct an optimization algorithm which uses only 
%PC oracle, (ii) Evaluate the convergence rate of that algorithm. In the next section, we give 
%an answer in both demands. 
%\end{comment}
%%%%%%%%%%%%%%%%%%%%%%%%%%%%%%%%%%%%%%%%%%%%%%%%%%%%%%%%%%%%
%%%%%%%%%%%%%%%%%%%%%%%%%%%%%%%%%%%%%%%%%%%%%%%%%%%%%%%%%%%%

\section{Main Results}
\label{mai}
\subsection{Algorithm}
\begin{algorithm}[t]
   \caption{Block coordinate descent using PC oracle (BlockCD[$n$, $m$]) }
   \label{alg1}
\begin{algorithmic}
   \STATE {\bf Input:} initial point $\x_0 \in \mathbb{R}^n$, and accuracy in line search $\eta>0$. 
 \STATE {\bf Initialize:} set $t=0$. 
   \REPEAT
   \STATE Choose $m$ coordinates ${i_1},\ldots,{i_m}$ out of $n$ coordinates
 according to the uniform distribution. 
   \STATE{\bf (Direction estimate step)}
   %\STATE{\bf In parallel on $R$ processors}
 \STATE[Step~D-1] Solve the one-dimension optimization problems
 \begin{align}
  \label{eqn:one_dim_opt_BlockCD}
  \min_{\alpha\in\Rbb} f(\x_{t} + \alpha \e_{i_k}),\quad k=1,\ldots,m,
 \end{align}
 within the accuracy $\eta/2$ using the PC-based line search algorithm shown in Algorithm~\ref{alg2}, 
 where $\e_{i}$ denotes the $i$-th unit basis vector. 
 % using the PC oracle in Algorithm~\ref{alg2}. 
 % and get 
 Then, obtain the numerical solutions $\alpha_{t,i_k},k=1,\ldots,m$. 
 \STATE[Step~D-2] Set $\d_{t} = \sum_{k=1}^{m} \alpha_{t, i_k}\e_{i_k}$. 
 If $\d_t$ is the zero vector, add $\eta/2$ to $d_{i_1}$. 
 \STATE{\bf (Search step)}
 \STATE[Step~S-1] 
 Apply Algorithm~\ref{alg2}
 to obtain a numerical solution $\beta_{t}$ of 
 \[
 \min_{\beta} f(\x_{t} + \beta \d_{t}/\|\d_t\|)
 \]
 within the accuracy $\eta$. 
 % by using PC oracle (see Algorithm~\ref{alg2}). 
   \STATE[Update] $\x_{t+1} = \x_{t} + \beta_{t} \d_{t}/\|\d_t\|;\ t\leftarrow{t+1}$. 
   \UNTIL{A stopping criterion is satisfied.}%{$\x_{t+1} = \x_{t}$}
   \STATE {\bf Output:} $\x_{t}$
\end{algorithmic}
\end{algorithm}

In Algorithm~\ref{alg1}, we propose a %parallel distributed 
DFO %derivative-free optimization 
algorithm based on the PC oracle. 
In our algorithm, 
$m$ coordinates out of $n$ elements are updated in each iteration to efficiently cope with high dimensional
problems. 
Algorithm~\ref{alg1} is referred to as $\mathrm{BlockCD}[n,m]$. 
The original PC algorithm is recovered by setting $m=1$. 
The PC oracle is used in the line search algorithm to solve one-dimensional
optimization problems; the detailed line search algorithm is shown in
Algorithm~\ref{alg2}. 

For $m=n$, the search direction $\d_t$ in Algorithm~\ref{alg1} approximates that of a
modified Newton method~\cite[Chap.\,10]{luenberger08:_linear_and_nonlin_progr}, as shown below. 
In Step~D-1 of the algorithm, one-dimensional optimization problems \eqref{eqn:one_dim_opt_BlockCD} are solved. 
%We employ the line search algorithm, proposed by Jamieson et al.~\cite{jamieson12:_query_compl_deriv_free_optim}, 
%using PC oracle (Algorithm~\ref{alg2}).  
Let $\alpha_{t,i}^*$ be the optimal solution of \eqref{eqn:one_dim_opt_BlockCD} with $i_k=i$. 
Then, 
$\alpha_{t,i}^*$ will be close to the numerical solution $\alpha_{t,i}$. The Taylor expansion of the 
objective function leads to
\begin{align*}
f(\x_t+\alpha\e_{i})
=
f(\x_t)+\alpha\e_{i}^T\nabla{f}(\x_t)+\frac{\alpha^2}{2}\e_{i}^T\nabla^2{f}(\x_t)\e_{i}+o(\alpha^2), 
\end{align*}
where $\nabla^2{f}(\x_t)$ is the Hessian matrix of $f$ at $\x_t$. 
When the point $\x_t$ is close to the optimal solution of $f(\x)$,
the optimal parameter $\alpha_{i,t}^*$ will be close to zero, implying
that the higher
order term $o(\alpha^2)$ in the above is negligible.
Hence, $\alpha_{t,i}$ is approximated by the optimal solution of the
quadratic approximation, i.e., 
$-(\nabla{f}(\x_t))_i/(\nabla^2f(\x_t))_{ii}$.
As a result, %in Algorithm~\ref{alg1} with $m=n$, 
the search direction in $\mathrm{BlockCD}[n,n]$ is approximated by
$-({\mathrm{diag}}(\nabla^2{f}(\x_t)))^{-1}\nabla{f}(\x_t)$, where
$\mathrm{diag}(A)$ denotes the diagonal
matrix, the diagonal elements of which are those of the square matrix~$A$.
In the modified Newton method, the Hessian matrix in the Newton method
is replaced with a positive definite matrix to
reduce the computation cost. Using only the diagonal part of the
Hessian matrix is a popular choice in the modified
Newton method.
\begin{algorithm}[t]
   \caption{line search algorithm using PC oracle \cite{jamieson12:_query_compl_deriv_free_optim}}
   \label{alg2}
\begin{algorithmic}
   \STATE {\bf Input:} current solution $\x_t \in \mathbb{R}^n$, search direction $\d \in \mathbb{R}^n$ 
   and accuracy in line search $\eta>0$. 
 \STATE {\bf Initialize:} set $\alpha_0 = 0$,  $\alpha^{+}_0 = \alpha_0 + 1$, $\alpha^{-}_0 = \alpha_0 - 1$, $k=0$. 
   \STATE {\bf [Step1]} 
    \IF {$O_{f}(\x_t, \x_t + \alpha^{+}_0 \d) > 0$ and $O_{f}(\x_t, \x_t + \alpha^{-}_0 \d) < 0$}
   \STATE $\alpha^{+}_0 \leftarrow 0$
   \ELSIF{$O_{f}(\x_t, \x_t + \alpha^{+}_0 \d) < 0$ and $O_{f}(\x_t, \x_t + \alpha^{-}_0 \d) > 0$}
   \STATE $\alpha^{-}_0 \leftarrow 0$
   \ENDIF
 \STATE {\bf [Step2]} (double-sign corresponds)
   \WHILE{$O_{f}(\x_t, \x_t + \alpha^{\pm}_k \d) < 0$}
   \STATE $\alpha^{\pm}_{k+1} \leftarrow 2\alpha^{\pm}_k$, $k \leftarrow k+1$
   \ENDWHILE
 \STATE {\bf [Step3]} 
   \WHILE{$|\alpha^{+}_k - \alpha^{-}_k| > \eta/2$}
   \IF{$O_{f}(\x_t + \alpha_k \d, \x_t + \frac{1}{2}(\alpha_k + \alpha^{+}_k) \d) < 0$} 
   \STATE $\alpha_{k+1} \leftarrow \frac{1}{2}(\alpha_k + \alpha^{+}_k)$, 
   			$\alpha^{+}_{k+1} \leftarrow \alpha^{+}_{k}$, $\alpha^{-}_{k+1} \leftarrow \alpha_{k}$
   \ELSIF{$O_{f}(\x_t + \alpha_k \d, \x_t + \frac{1}{2}(\alpha_k + \alpha^{-}_k) \d) < 0$} 
   \STATE $\alpha_{k+1} \leftarrow \frac{1}{2}(\alpha_k + \alpha^{-}_k)$, 
   			$\alpha^{-}_{k+1} \leftarrow \alpha^{-}_{k}$, $\alpha^{+}_{k+1} \leftarrow \alpha_{k}$
   \ELSE 
   \STATE (double-sign corresponds)
   \STATE $\alpha_{k+1} \leftarrow \alpha_{k}$, 
   			$\alpha^{\pm}_{k+1} \leftarrow \frac{1}{2}(\alpha_k + \alpha^{\pm}_k)$
   \ENDIF
   \ENDWHILE
   \STATE {\bf Output:} $\alpha_{t}$
\end{algorithmic}
\end{algorithm}

Figure~\ref{fg1} 
demonstrates an example of the optimization process of both the original PC algorithm and our algorithm.
The original PC algorithm updates the numerical solution along a randomly chosen
coordinate in each iteration. % with respect to a coordinate,  
Hence, many iterations are required to get close to the optimal solution. 
On the other hand, in our algorithm, a solution can move along a oblique direction. 
Therefore, our algorithm can get close to the optimal solution with less iterations than the original PC algorithm.

\begin{figure}[t]
 \begin{center}
 \includegraphics[scale=0.4]{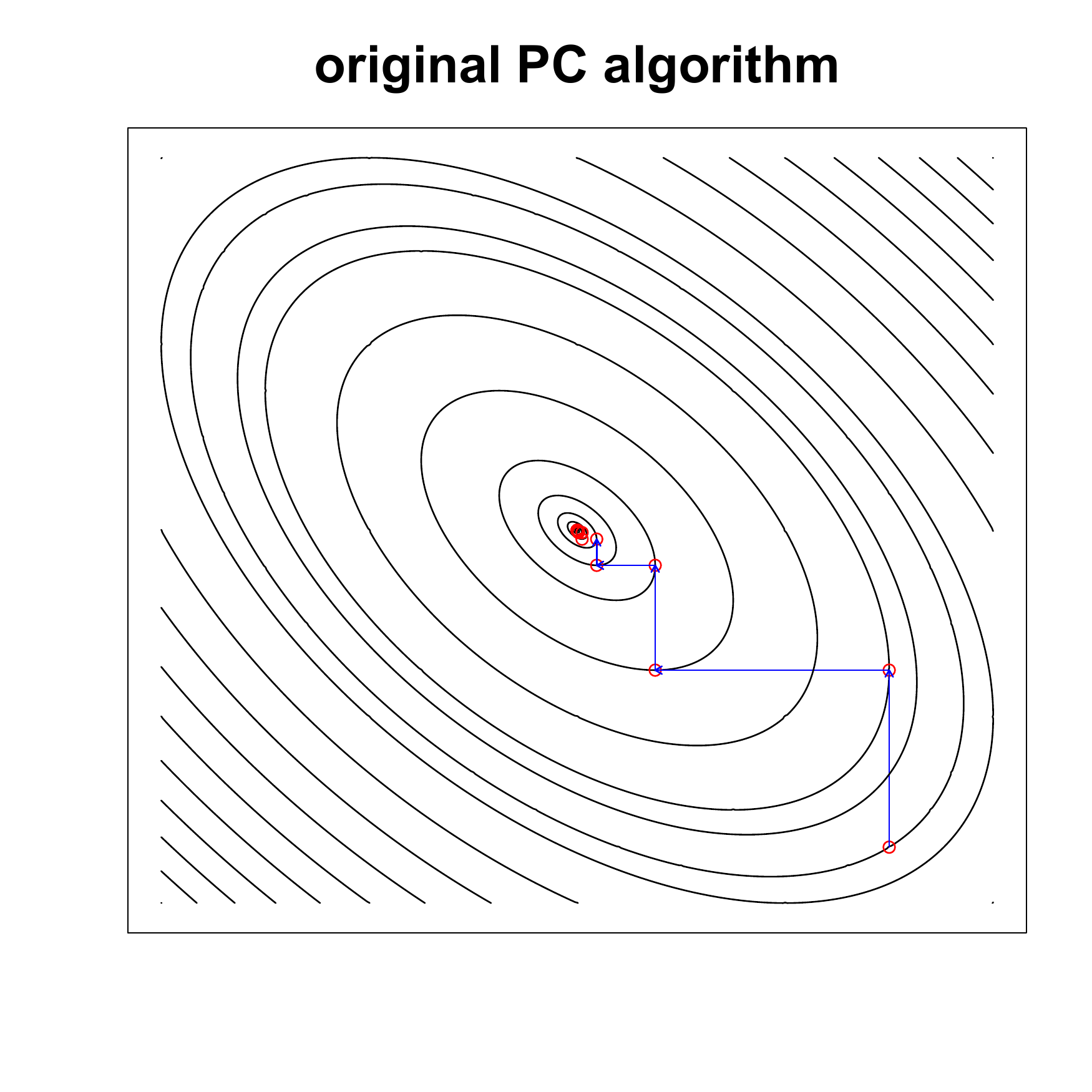}\quad 
 \includegraphics[scale=0.4]{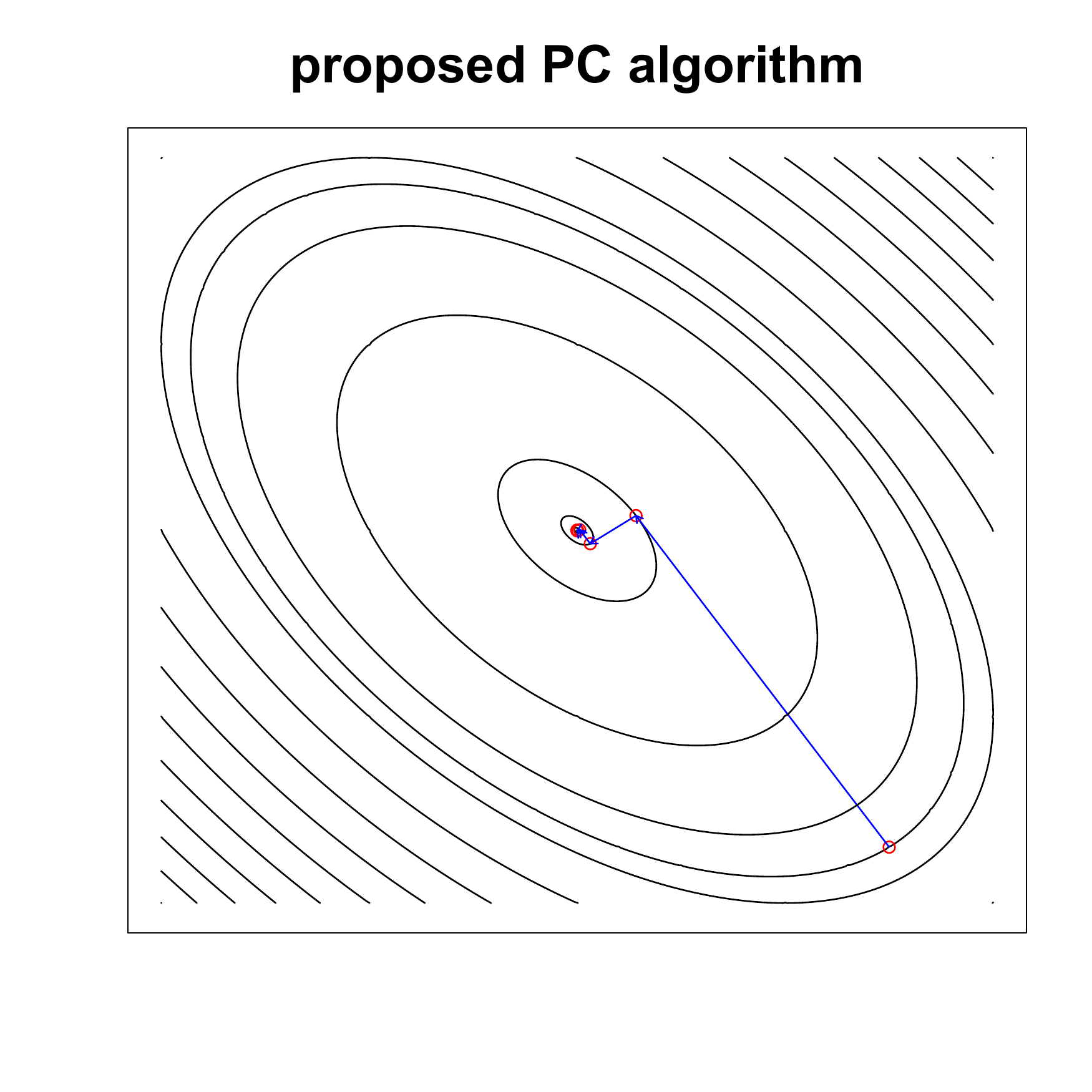}
 \end{center}
 \vspace*{-10mm}
\caption{A behavior of the algorithms on the contour of the quadratic objective function 
$x_{1}^2 + x_{2}^2 + x_{1}x_{2}$ with same initialization. 
Left panel: Jamieson {et al.}'s original PC algorithm. 
Right panel: proposed algorithm. }
\label{fg1}
\end{figure}

%\end{comment}
%%%%%%%%%%%%%%%%%%%%%%%%%%%%%%%%%%%%%%%%%%%%%%%%%%%%
\subsection{Convergence Properties of our Algorithm under Deterministic Oracle}
We now provide an upper bound of the convergence rate of our algorithm using the deterministic PC oracle \eqref{eqP4}. 
%That is, we treat the situation that the PC oracle $O_{f}$ satisfies (\ref{eqP4}). 
Let us denote the minimizer of $f$ as $\x^*$. 
%\begin{eqnarray*}
%\x^*=\argmin_{\x \in \mathbb{R}^n} f(x). 
%\end{eqnarray*}
%Moreover, $\x_{t}$ and $\hat{\x}_{Q}$ represent the output of the proposal algorithm after $t$ iteration 
%and the output of the proposal algorithm after $Q$ pairwise comparison queries respectively.
%The total number of pairwise comparison queries after $t$ iterations is represented as 
%\begin{equation}
%\label{eq:que}
%Q_{t} = \sum_{i=1}^{t} \left( \sum_{j=1}^{n} {\bar{Q}}_{ij}(\eta) + {\bar{Q}}_{i}(\eta) \right)
%\end{equation}
%where both ${\bar{Q}}_{ij}(\eta)$ and ${\bar{Q}}_{i}(\eta)$ represent the number of queries in the 
%line search at $i$-th iteration under the accuracy $\eta$. ${\bar{Q}}_{ij}(\eta)$ represent the number 
%of queries of $j$-th direction and ${\bar{Q}}_{i}(\eta)$ represents that of $\d_{t}$ direction.
%Then, $\x_{t}$ coincides with $\hat{\x}_{Q_{t}}$.   
%Let ${\cal F}_{\sigma,L}(\R^n)$ be the set of $\sigma$-strongly convex and $L$-gradient Lipschitz functions on $\R^n$.
\begin{theo}
\label{upper}
 Suppose $f\in\mathcal{F}_{\sigma,L}(\mathbb{R}^n)$, and 
 define $\gamma$ and $\varepsilon$ be
 \begin{align*}
  \gamma=\frac{\sigma/L}{53}\left(\frac{1-\sqrt{1-\sigma/L}}{1+\sqrt{1-\sigma/L}}\right)^2,\quad
    \varepsilon=
  \frac{8nL^2}{\sigma}\bigg(1+\frac{n}{m\gamma}\bigg)\eta^2. 
  %  \frac{L(1+\frac{n}{m\gamma})}{\min\{\frac{\sigma/L}{8n},\,\gamma\}}. 
  %  \varepsilon=\eta^2\frac{L(1+\frac{n}{m\gamma})}{\min\{\frac{\sigma/L}{8n},\,\gamma\}}. 
 \end{align*}
 Let us define $T_0$ be
 \begin{align}
  \label{eqn:def_T0}
  T_0=
  \bigg\lceil
  \frac{n}{m\gamma}\log\frac{(f(\x_0)-f(\x^*))(1+\frac{n}{m\gamma})}{\varepsilon}
  \bigg\rceil. 
 \end{align}
 For $T\geq{T_0}$, we have $\mathbb{E}[f(\x_T)-f(\x^*)]\leq{}\varepsilon$, 
 where the expectation is taken with respect to the random choice of coordinates $i_1,\ldots,i_k$ to be updated in 
 $\mathrm{BlockCD}[n,m]$. 
\end{theo}
The proof of Theorem~\ref{upper} is given in~\ref{proof.upper}. 
%%That is, for $\kappa = 1$, 
Note that any monotone transformation of the objective function does not affect the output of the deterministic PC
oracle. Hence, the theorem above holds even for the function $f(\x)$ such that the
composite function with a monotone function is included in $\mathcal{F}_{\sigma,L}(\R^n)$. 
%Note that it is sufficient for optimization to assume that a monotone transformation of $f$ is an element of 
%$\mathcal{F}_{\sigma, L}(\R^n)$, even if $f$ itself is not. 

%\memo{====reconsider the definition of $\varepsilon$====}

%Theorems \ref{upper} and \ref{lower} implies that the modified PC algorithm achieves near optimal rate with respect to the 
%increment of the dimension $n$ except for poly-logarithmic factors.
%Thus, when the objective function is too complex to obtain its gradient and has many decision variables,
%the modified PC algorithm is expected to be useful.

%%%%%%%%%%%%%%%%%%%%%%%%%%%%%%%%%%%%%%%%%%%%%%%%%%%%%%%%
\subsection{Query Complexity}
%Let $Q$ be the number of queries and let $\x_t$ and $\hat{\x}_{Q}$ be the output of our algorithm 
%after $t$ iterations and after $Q$ pairwise comparison queries respectively.
Let $\hat{\x}_{Q}$ be the output of BlockCD$[n,m]$ after $Q$ pairwise comparison queries. 
To solve the one dimension optimization problem within the accuracy $\eta/2$, 
the sufficient number of the call of PC-oracle is 
%it is sufficient we need to call the required number of the call of PC-oracle is 
\begin{align*}
 K_0=2\log_2\frac{2^{10}L(f(\x_0)-f(\x^*))}{\sigma^2\eta^2}, %=c_0+\log\frac{1}{\eta}
\end{align*}
as shown in \cite{jamieson12:_query_compl_deriv_free_optim}. 
Hence, if the inequality $Q\geq{}T_0K_0(m+1)$ 
 %=\frac{n}{\gamma}\log\frac{1+n/m\gamma}{\varepsilon}
holds, 
Theorem~\ref{upper} assures that 
the numerical solution $\hat{\x}_{Q}$ based on $Q$ queries satisfies
\[
\mathbb{E}[f(\hat{\x}_{Q})-f(\x^*)]\leq\varepsilon. 
\]
When $n/\varepsilon$ is sufficiently large, we have
\begin{align*}
 &\phantom{=} T_0K_0(m+1)\\
 &=
 (m+1) 
 \bigg\lceil
  \frac{n}{m\gamma}\log\frac{(f(\x_0)-f(\x^*))(1+\frac{n}{m\gamma})}{\varepsilon}
 \bigg\rceil\cdot
 \bigg\lceil
 2\log_2\frac{2^{10}L(f(\x_0)-f(\x^*))}{\sigma^2\eta^2}
 \bigg\rceil\\
 &=
 (m+1)  \bigg\lceil
 \frac{n}{m\gamma}\log\frac{\Delta_0(1+\frac{n}{m\gamma})}{\varepsilon}
 \bigg\rceil\cdot
  \bigg\lceil
 2\log_2\frac{2^{13}(L/\sigma)^3\Delta_0n(1+\frac{n}{m\gamma})}{\varepsilon}
 \bigg\rceil\\
 &\leq
 c_0 n\left(\log\frac{n}{\varepsilon}\right)^2
 \leq 
 c_0 n(\log{n})^2\left(\log\frac{1}{\varepsilon}\right)^2, 
\end{align*}
%for $\log{n}\geq2$ and $\log(1/\varepsilon)\geq2$, 
where $c_0$ is a constant depending  on $L/\sigma$ and $\Delta_0=f(\x_0)-f(\x^*)$. 
The last inequality holds if $\log{n}$ and $\log(1/\varepsilon)$ are both greater 
than $2$. %Let us define $Q$ as the rightmost side of the above inequality. 
Eventually we have
\begin{align}
\label{eq:quecom}
 \Ebb[f(\hat{\x}_{Q})-f(\x^*)]\leq{}
 \exp\left\{-\frac{c}{\log{n}}\sqrt{\frac{Q}{n}}\right\}, 
% \varepsilon\sim{}e^{-\frac{c}{\log{n}}\sqrt{\frac{Q}{n}}}, 
 % \varepsilon\sim{}e^{-\widetilde{O}(\sqrt{Q/n})}
 % \approx\exp\left\{-C\sqrt{\frac{Q}{n\log{n}}}\right\}, 
 %\frac{n^2}{m}e^{-C\sqrt{Q/n}}. 
\end{align}
where $c=1/\sqrt{c_0}$. % is a constant depending on $\sigma,L$ and $f(\x_0)-f(\x^*)$. 
%is a constant depends on $\sigma,L$ and $f(\x_0)-f(\x^*)$. 
The above bound is of the same order of the convergence rate 
for the original PC algorithm %PC-based coordinate descent method~\cite{jamieson12:_query_compl_deriv_free_optim}
up to polylog factors. On the other hand, a lower bound presented
in~\cite{jamieson12:_query_compl_deriv_free_optim} is of order $e^{-cQ/n}$ with a positive constant $c$ up to polylog
factors, when the PC oracle with $\kappa=1$ is used. 

%However, in contrast with Jamieson et al.'s stochastic PC algorithm, 
%our algorithm is a deterministic algorithm and thus 
%the output of our algorithm does not fluctuate with respect to a fixed initial point.
%On the other hand, 
%the lower bound is derived for a stochastic PC oracle including the deterministic one 
%as follows. 
%\begin{theo}[\cite{jamieson12:_query_compl_deriv_free_optim}]
%%[Jamieson {\it et al.} 2012]
%\label{lower}
%For $n\ge8$ and sufficiently large $Q$,
%\begin{equation}
%\inf_{\hat{\x}_Q}\sup_{f\in{\cal F}_{\sigma,L}(\R^n)} {\mathbb E}[f(\hat{\x}_Q)-f(\x^*)]
%\ge c_1\exp\left(-c_2\frac{Q}{n}\right)
%\end{equation}
%where the infimum is over the collection of possible estimators of $\x^*$ using at most $Q$ queries to a PC oracle.
%The constants $c_1$ and $c_2$ depend on function class parameters $\sigma$ and $L$.
%\end{theo}

In Theorem \ref{upper}, it is  assumed that the objective function is strongly 
convex and gradient Lipschitz.
In a realistic situation, 
we usually do not have the knowledge of the class parameters $\sigma$ and $L$ of the unknown objective function.
Moreover, strong convexity and gradient Lipschitzness  on the whole space $\R^n$ is too strong.
%In fact, 
In the following corollary, we relax the assumption in Theorem \ref{upper} and prove the
convergence property of our algorithm without strong convexity and strong smoothness. 
%we modify the modified PC algorithm as follows.
%Then we have the following convergence theorem without strong convexity and gradient Lipschitzness.
%kuma
\begin{cor}
\label{convergence}
Let $f:\R^n\to\R$ be a twice continuously differentiable convex function  with non-degenerate Hessian on $\R^n$ and 
$\x^*$ be a minimizer of $f$.
Then, there is a constant $c$ such that the output $\hat{x}_Q$ of BlockCD[n,m] satisfies (\ref{eq:quecom}).
%For the output $\hat{\x}_Q$ of our algorithm with accuracy $\epsilon>0$,
%the difference $f(\hat{\x}_Q)-f(\x^*)$ sub-exponentially goes to less than $\epsilon$ with respect to the number 
%$Q$ of queries. 
\end{cor}
The proof of Corollary~\ref{convergence} is given in~\ref{proof.convergence}. 
%%%%%%%%%%%%%%%%%%%%%%%%%%%%%%%%%%%%%%%%%%%%%%%%%%%%%%%%
\subsection{Generalization to Stochastic Pairwise Comparison Oracle}
\begin{algorithm}[tb]
   \caption{Repeated querying subroutine 
   (\cite{jamieson12:_query_compl_deriv_free_optim,kaariainen2006active})}
   \label{alg3}
\begin{algorithmic}
   \STATE {\bf Input:} $\x, \y \in \mathbb{R}^n$, $p = {\rm Pr} [O_{f}(\x, \y) = {\rm sign} \{f(\y) - f(\x) \}]$, $\delta > 0$
   \STATE {\bf Initialize:}
    set $n_0 = 1$ and toss the coin with probability $p$ of heads once. 
   \FOR{$k=0, 1, ...$}
     \STATE{ $p_k = \mbox{frequency of heads in all tosses so far}$} 
   \STATE{$I_k = \left[p_k - \sqrt{\frac{(k+1)\log(2/\delta)}{2^k}}, p_k + \sqrt{\frac{(k+1)\log(2/\delta)}{2^k}} \right]$} 
   \IF{$\frac{1}{2} \not \in I_k$} 
   \STATE{\bf{break}}
   \ELSE 
   \STATE{toss the coin $n_k$ more times, and set $n_{k+1} = 2n_k$.}
   \ENDIF
  \ENDFOR 
  \IF{$p_k + \sqrt{\frac{(k+1)\log(2/\delta)}{2^k}} \le \frac{1}{2}$}
  \STATE {{\bf return}} $-1$
  \ELSE
  \STATE{{\bf return}} $+1$
  \ENDIF
\end{algorithmic}
\end{algorithm}

%Jamieson et al. \cite{jamieson12:_query_compl_deriv_free_optim} guaranteed the reliability of 
%line search algorithm under stochastic PC oracle (\ref{eqP3}) by using Algorithm~\ref{alg3} and the following 
%lemma~\ref{sto_line} in \cite{kaariainen2006active}. 

In stochastic PC oracle, one needs to ensure that the correct information is obtained in high probability. 
In Algorithm~\ref{alg3}, the query $O_f(\x,\y)$ is repeated under the stochastic PC oracle. 
The reliability of line search algorithm based on stochastic PC oracle \eqref{eqP3} was investigated by 
\cite{jamieson12:_query_compl_deriv_free_optim,kaariainen2006active}. 
\begin{lem}
 [\cite{jamieson12:_query_compl_deriv_free_optim,kaariainen2006active}]
 \label{sto_line}
 For any $\x, \y \in \mathbb{R}^n$ with $p = {\rm Pr} [O_{f}(\x, \y) = {\rm sign} \{f(\y) - f(\x) \}]$, 
 the repeated querying subroutine in Algorithm~\ref{alg3} correctly identifies 
 the sign of $\Ebb[O_{f}(\x, \y)]$ with probability $1- \delta$, and requests no more than 
\begin{equation}
\label{eq:sto_que_line}
\frac{\log{2/\delta}}{4|1/2 - p|^2} \log_{2} \left(\frac{\log{2/\delta}}{4|1/2 - p|^2} \right)
\end{equation}
queries. %PC oracle. 
\end{lem} 
It should be noted here that, in this paper,  ${\rm sign} \{E[O_f(x,y)]\}= {\rm sign} \{f(y)-f(x)\}$ always holds because 
$p>1/2$ from (\ref{eqP3}). 
In \cite{jamieson12:_query_compl_deriv_free_optim}, a modified line search algorithm using a ternary search instead of 
bisection search was proposed to lower bound $|1/2 - p|$ in Lemma \ref{sto_line} for arbitrary $\x, \y \in \mathbb{R}^n$. 
Then, one can find that the total number of queries required by a repeated querying subroutine algorithm is at most  
$\tilde{O} \left(\frac{\log{1/\delta}}{\eta^{4(\kappa - 1)}} \right)$, where $\eta$ is an accuracy of line search.
The query complexity of the stochastic PC oracle is obtained from that of the deterministic PC oracle. 
Suppose that one has $Q_0$ responses from the deterministic PC oracle. To
obtain the same responses from the stochastic PC oracle with
probability more than $1 - \delta$, one needs more than $\tilde{O}(Q_0 \eta^{-4(\kappa - 1)}
\log(\frac{Q_{0}}{\delta}))$ queries. 
%Jamieson {\it et al.} \cite{jamieson12:_query_compl_deriv_free_optim} derived the upper bound of 
%convergence rate of PC algorithm under the stochastic oracle. 
%We show here that the similar bound can be derived for modified PC algorithm. 
From the above discussion, we have the following upper bounds for stochastic setting: 
%query complexity 
%\begin{equation}
 %\Ebb[f(\hat{\x}_{Q})-f(\x^*)] \le c_1 n \left(1 + \frac{n}{m \gamma} \right) \left(\frac{n}{Q} \right)^{\frac{1}{2(\kappa - 1)}}, 
%\label{eq:sto_compl}
%\end{equation}
\begin{equation}
 \Ebb[f(\hat{\x}_{Q})-f(\x^*)] \le 
  \begin{cases}
   \displaystyle
   \exp \left \{-\frac{c_{1}}{\log{n}} \sqrt{\frac{Q}{n}} \right \}, & \kappa = 1,  \vspace*{2mm} \\ 
   \displaystyle
   c_{2} \displaystyle \frac{n^2}{m}  \left(\frac{n}{Q} \right)^{1/(2\kappa - 2)}, & \kappa > 1, 
  \end{cases}
  \label{eq:sto_compl}
\end{equation}
where $c_1$ and $c_2$ are constant depending  on $L/\sigma$ and $f(\x_0) - f(\x^*)$ as well as $1/ \delta$ poly-logarithmically. 
If $m$ and $n$ are of the same order in the case of $\kappa > 1$, 
the bound \eqref{eq:sto_compl} %with respect to $n$ and $Q$ 
coincides with that shown in Theorem~2 of \cite{jamieson12:_query_compl_deriv_free_optim}.

\section{Numerical Experiments}
\label{num}
In this section, we present numerical experiments in which the proposed method 
in Algorithm~\ref{alg1} was mainly compared with the Nelder-Mead 
algorithm~\cite{nelder65:_simpl_method_funct_minim} and 
the original PC algorithm~\cite{jamieson12:_query_compl_deriv_free_optim}, i.e.,
BlockCD$[n,1]$ of Algorithm~\ref{alg1}. 
Here, the PC oracle was used in all the optimization algorithms. 
In BlockCD$[n,m]$ with $m\geq2$, one can execute the line search algorithm to each axis separately. 
Hence, the parallel computation is directly available to find the components of the search
direction $\d_t$. 
Also, we investigated the computation efficiency of the parallel implementation of our method. 
The numerical experiments were conducted on AMD Opteron Processor 6176 (2.3GHz) with 48 cores, 
running Cent OS Linux release 6.4. We used the R language~\cite{team14:_r} with {\tt snow}
library for parallel statistical computing.

\subsection{Two Dimensional Problems}
It is well-known that the Nelder-Mead method efficiently works in low dimensional problems. 
Indeed, in our preliminary experiments for two dimensional optimization problems, 
the Nelder-Mead method showed a good convergence property compared to the other methods such as 
BlockCD$[2,m]$ with $m=1,2$. 
Numerical results are presented in Figure~\ref{fig:2_dim_problems}. 
We tested optimization methods on the quadratic function $f(x)=x^TAx$, and  two-dimension Rosenbrock function,
$f(x)=(1-x_1)^2+100(x_{2}-x_1^2)^2$, where the matrix $A$ was a randomly generated 2 by 2 positive definite matrix.  
In two dimension problems, we do not use the parallel implementation of our method, 
since clearly the parallel computation is not efficient that much. 
The efficiency of parallel computation is canceled by the communication overhead. 
In our method, the accuracy of the line search is fixed to a small positive number
$\eta$. Hence, the optimization process stops on the way to the optimal solution, 
as shown in the left panel of Fig.~\ref{fig:2_dim_problems}. 
On the other hand, the Nelder-Mead method tends to converge to the optimal solution in high accuracy.
In terms of the convergence speed for the optimization of two-dimensional quadratic function, there is no
difference between the Nelder-Mead method and BlockCD method, until the latter stops due to the limitation of
the numerical accuracy. Even in non-convex Rosenbrock function, the Nelder-Mead method works well compared to
the PC-based BlockCD algorithm. 

\begin{figure}[tp]
 %\hspace*{-3mm}
 \begin{center}
 \begin{tabular}{cc}
  \includegraphics[scale=0.4]{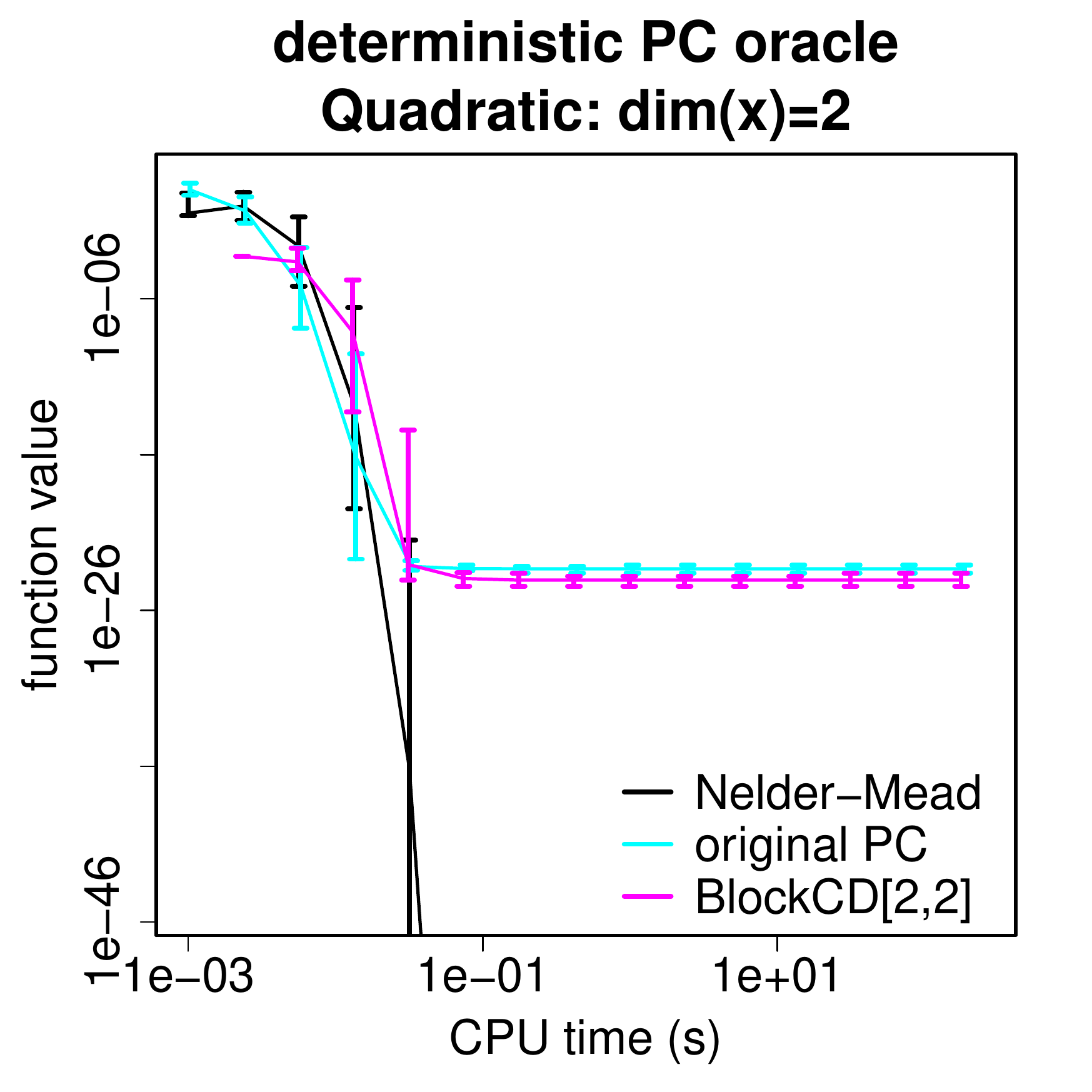}&
  \includegraphics[scale=0.4]{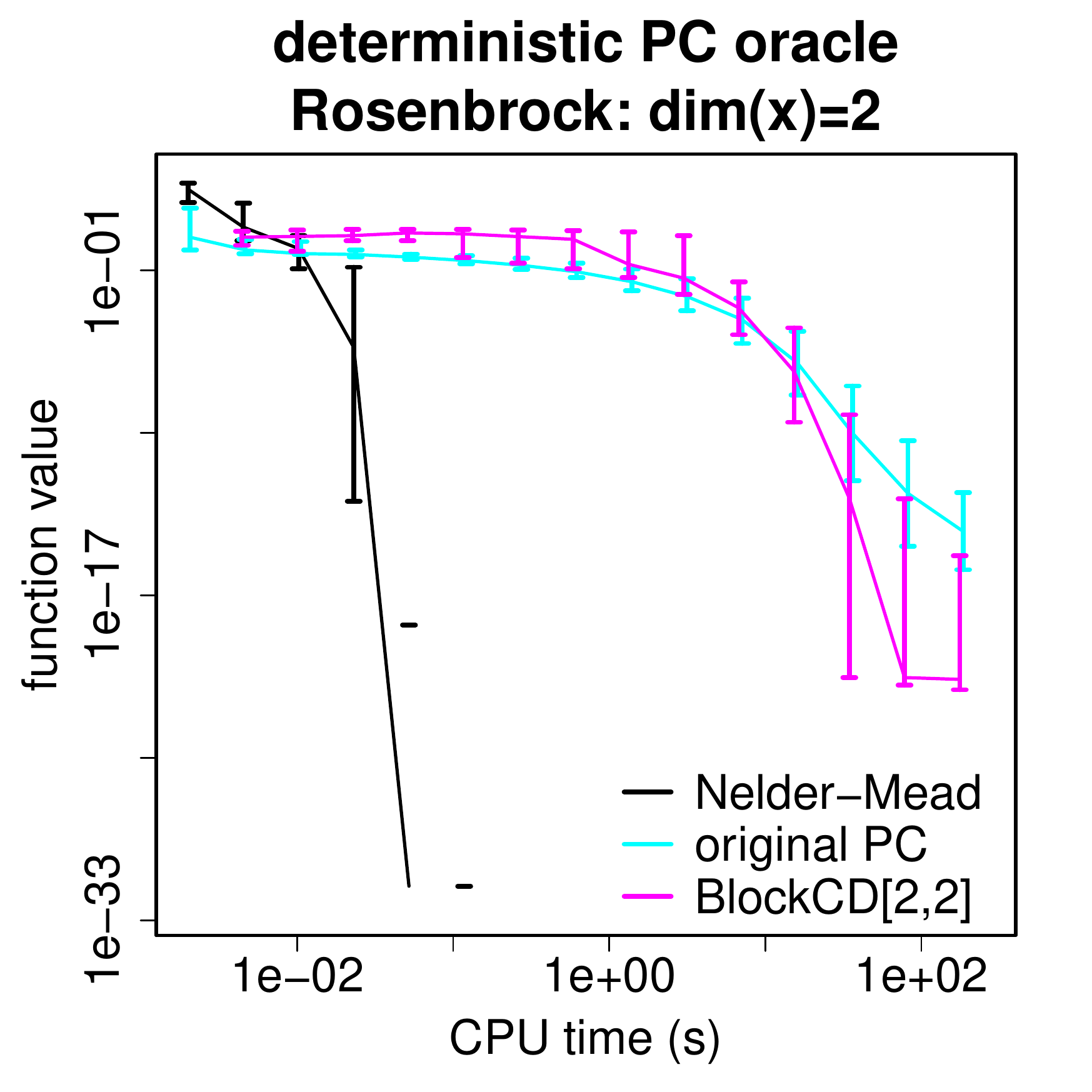}
 \end{tabular}
 \caption{
  Left panel: 2-dimension quadratic function. Right panel: 2-dimension Rosenbrock function. 
  The Nelder-Mead method, BlockCD$[2,1]$, and BlockCD$[2,2]$ are compared. 
  For each algorithm, the median of the function value is depicted to the CPU time (s). 
  The vertical bar shows the percentile $30\%$ to $70\%$.}
 \label{fig:2_dim_problems}
 \end{center}
\end{figure}

\subsection{Numerical Experiments of Parallel Computation}
In high dimensional problems, however, the performance of the Nelder-Mead method is easily degraded as 
reported by several authors; see \cite{gao12:_implem_nelder_mead_simpl_algor_adapt_param} and references
therein. In the below, we focus on solving moderate-scale optimization problems. 

%\memo{===stochastic Nelder-Mead is not run: how to explain??===}
In experiments using the PC oracle, BlockCD$[n,m]$ with $m\geq2$ and its parallel
implementations were compared with the Nelder-Mead method and the original PC algorithm, 
i.e., BlockCD$[n,1]$. In each iteration of BlockCD$[n,m]$ with $m\geq2$, $m+1$ runs of the
line search were required. In the parallel implementation, tasks of line search except the
search step in Algorithm~\ref{alg1} were almost equally assigned to each core in processors. 
In the below, the parallel implementation of BlockCD$[n,m]$ is referred to as
parallel-BlockCD$[n,m]$. Suppose that $c$ cores are used in parallel-BlockCD$[n,m]$.  
Then, ideally, the parallel computation will be approximately $(m+1)/(m/c+1)\approx{c}$ times more
efficient than the serial processing, when $n$ and $m$ are much greater than $c$. 
Practically, however, the communication overhead among processors may cancel the effect of 
the parallel computation, especially in small-scale problems.  

\begin{figure}[tp]
 %\hspace*{-3mm}
 \begin{center}
 \begin{tabular}{cc}
  \includegraphics[scale=0.4]{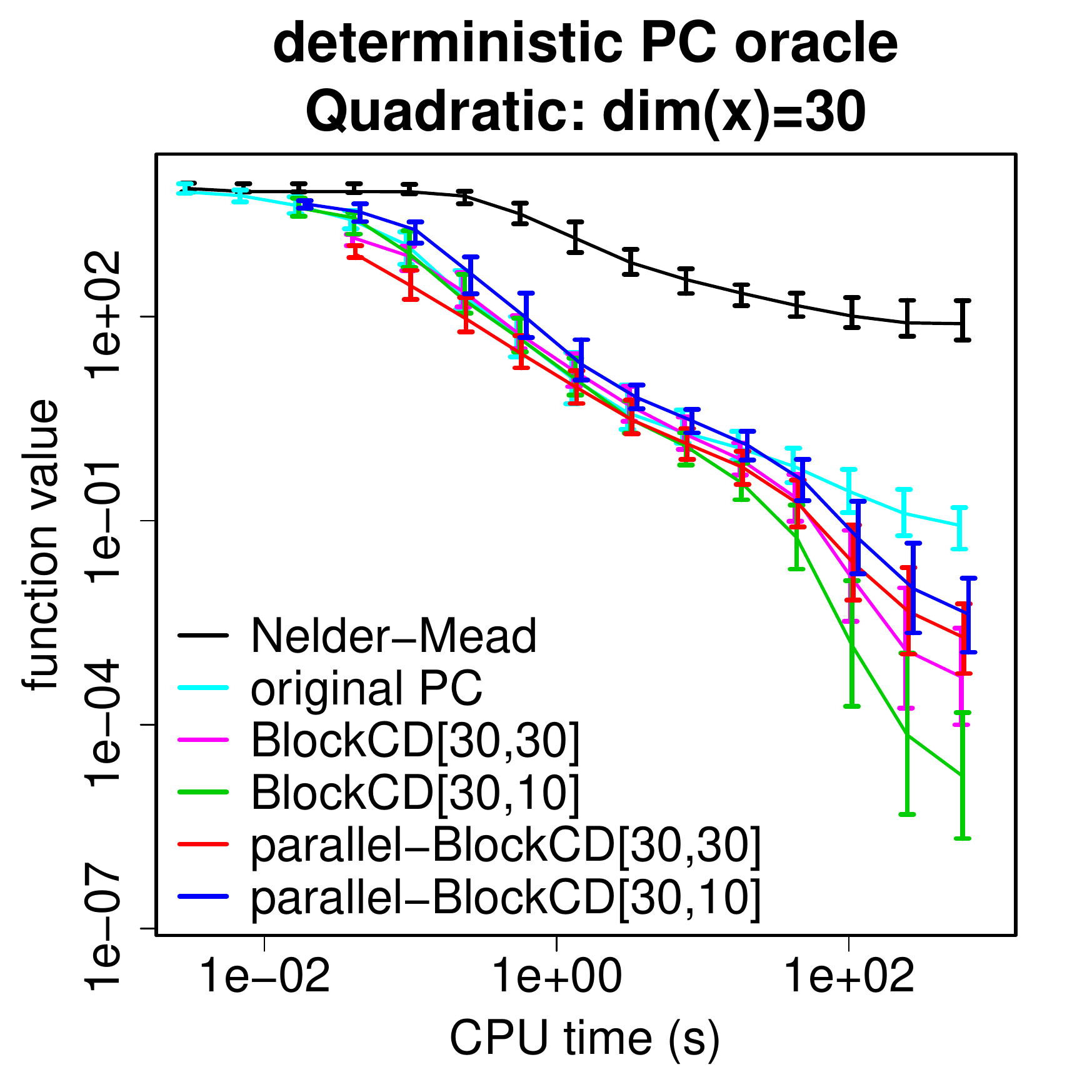}&
  \includegraphics[scale=0.4]{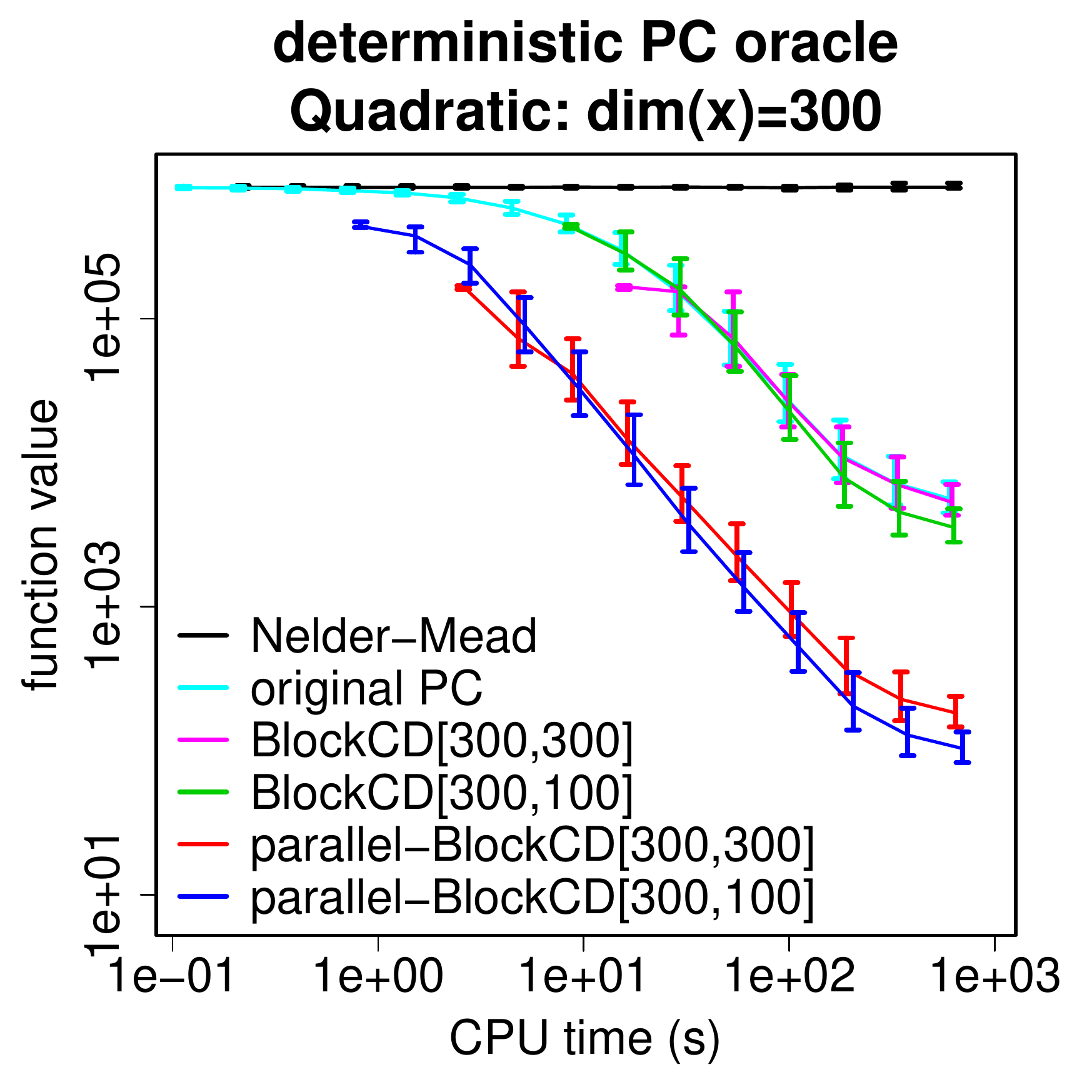}\\
  \includegraphics[scale=0.4]{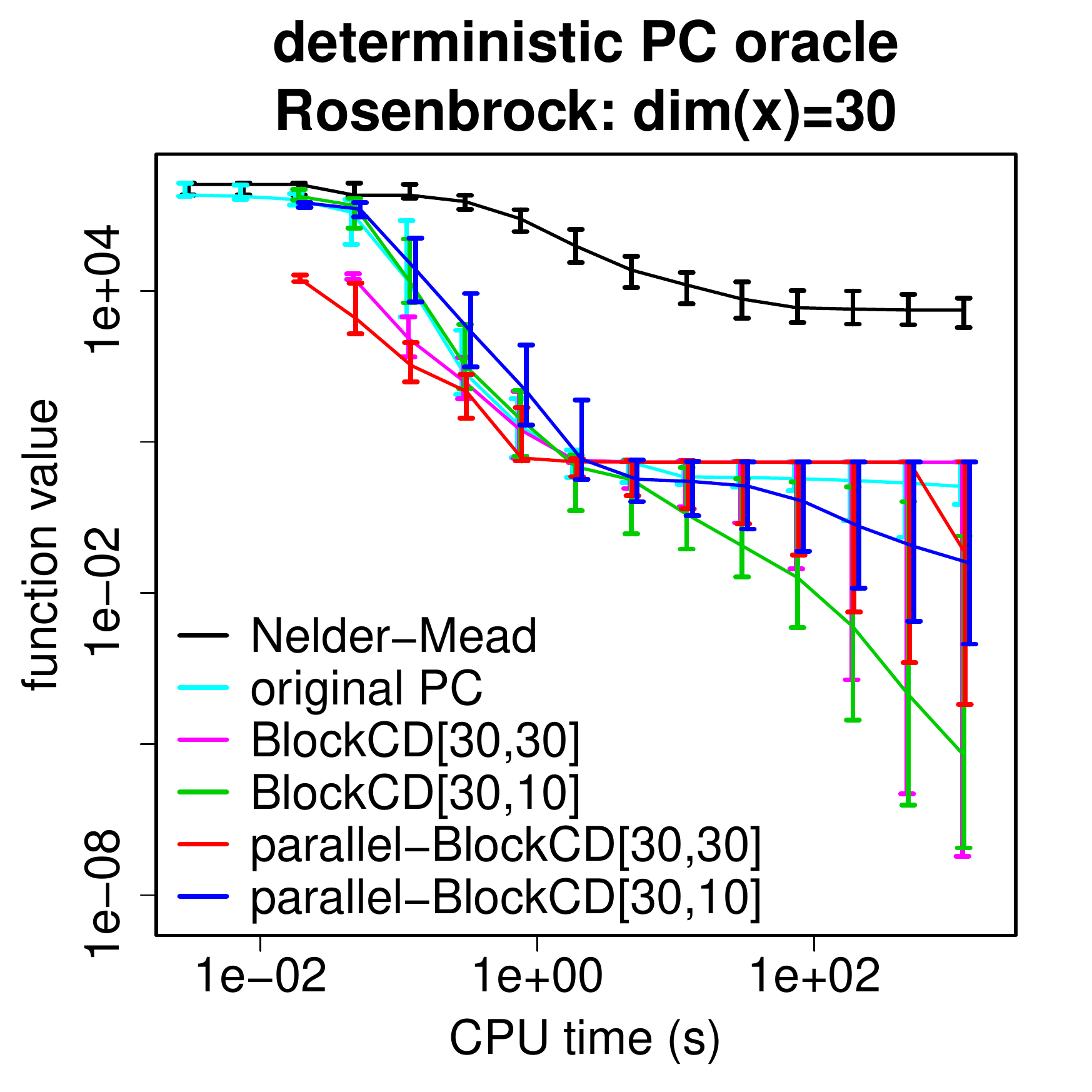}&
  \includegraphics[scale=0.4]{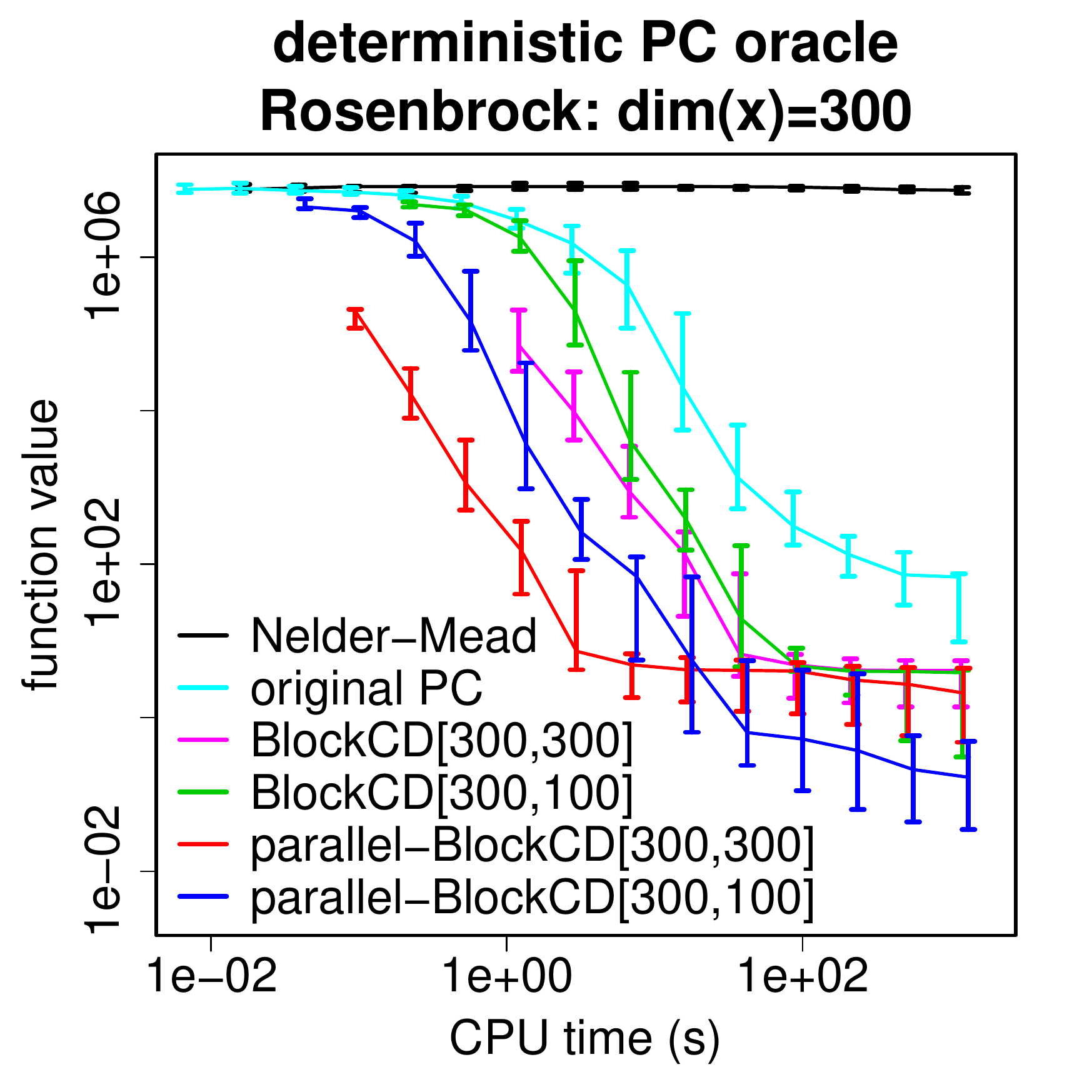}
 \end{tabular}
 \caption{
  Deterministic PC oracle is used in PC-based BlockCD algorithm. 
  Top panels: results in optimization of quadratic function. Bottom panels: results in
  optimization of Rosenbrock function. 
  The original PC algorithm,  BlockCD$[n,m]$ with $m=n$ and $m=n/3$, 
  and parallel-BlockCD$[n,m]$ with $m=n$ and $m=n/3$, 
  are compared for $n=30$ and $n=300$. 
  The median of the function value is shown to the CPU time (s). 
  The vertical bar shows the percentile $30\%$ to $70\%$. 
  }
 \label{fig:high_dim_problems}
 \end{center}
\end{figure}
We tested optimization methods on two $n$-dimensional optimization problems, i.e., 
the quadratic function $f(x)=x^TAx$, and Rosenbrock function, 
$f(x)=\sum_{i=1}^{n-1}[(1-x_i)^2+100(x_{i+1}-x_i^2)^2]$, where the matrix $A$ was a 
randomly generated $n$ by $n$ positive definite matrix. 
The quadratic function satisfies the assumptions in Theorem~\ref{upper}, while the
Rosenbrock function is not convex. 
We examine whether the proposed method is efficient even when the theoretical assumptions
are not necessarily assured. 
In each objective function, the dimension was set to $n=30$ or $300$. 
In all problems, the optimal value is zero. % and the optimal solutions is the zero vector. 
For each algorithm, the optimization was repeated 10 times using randomly chosen initial
points. 
%Then, we obtained many pairs consisting of the CPU time and function value. 
%Based on these data points, we computed 30 to 70 percentile of function values for each
%CPU time. 
According to~\cite{gao12:_implem_nelder_mead_simpl_algor_adapt_param}, 
we examined some tuning parameters for the Nelder-Mead algorithm, and we found that the
initial simplex does not significantly affect the numerical results in the present
experiments. Hence, the standard parameter setting of the Nelder-Mead method used
in~\cite{gao12:_implem_nelder_mead_simpl_algor_adapt_param} was used throughout the
present experiments. 

%\memo{====deterministic PC oracle:DONE====}
The numerical results using the deterministic PC oracle are presented in
Figure~\ref{fig:high_dim_problems}.  
For each algorithm, the median of function values in optimization process is depicted as the solid
line with 30 and 70 percentiles for each CPU time. 
The results indicate that the Nelder-Mead method does not efficiently work even for
30-dimensional quadratic function. 
The original PC algorithm and serially executed BlockCD$[n,m]$ were comparable. 
This result is consistent with \eqref{eq:quecom}. %Theorem~\ref{upper}.
When the deterministic PC oracle is used, the upper bound of the query complexity is
independent of $m$. %almost the same order in terms of $Q$ in both methods.   
As for the efficiency of the parallel computation, the parallel-BlockCD$[n,m]$
outperformed the competitors in 300 dimensional problems.  
In our experiments, the parallel implementation was about 15 times more efficient than the
serial implementation in CPU time. 
For large-scale problems, the communication overhead is canceled by the efficiency of the
parallel computation. In our approach, the parallel computation is easily conducted
without losing the convergence property proved in Theorem~\ref{upper}. 

%\memo{====stochastic PC oracle:DONE====}
Also, we conducted optimization using the stochastic PC oracle. The results are shown in
Fig.~\ref{fig:high_dim_problems_stochasticPC}. The parameter in the stochastic PC oracle
was set to $\kappa=2, \delta_0=0.3$ and $\mu=0.01$. 
Thus, the difference of two function values affects the probability that the oracle returns the
correct sign. According to Lemma~\ref{sto_line}, the query was repeated at each point so
that the probability of receiving the correct sign was greater than $1-\delta$ with
$\delta=0.01$. 
As shown in the results of $\mathrm{BlockCD}[300,100]$ and $\mathrm{BlockCD}[300,300]$, 
the serial implementation of $\mathrm{BlockCD}[n,m]$ for a large $m$ was extremely
inefficient. Indeed, the right panels of Fig.~\ref{fig:high_dim_problems_stochasticPC}
indicate that an iteration of $\mathrm{BlockCD}[300,300]$ takes a long time. 
Also, the convergence rate of the original PC algorithm was slow, though the computational
cost of each iteration was not high. When the stochastic PC oracle was used, the parallel
implementation of BlockCD$[n,m]$ achieved fast convergence rate compared with the other
algorithms in CPU time. 

%%There is a difference between 
%\memo{========not good result. remove? ============}
%Let us investigate the difference between deterministic PC oracle and stochastic PC oracle. 
%In the optimization of quadratic functions using deterministic PC oracle, the computation time of
%parallel-BlockCD$[n,m]$ seems not to depend on $m$ so much. 
%As shown in the top panels in Fig.~\ref{fig:high_dim_problems_stochasticPC}, however, 
%the computation time of parallel-BlockCD$[n,m]$ using the stochastic PC oracle does depend 
%on $m$. These results matches to \memo{=================}, since the computation time is almost comparable to
%the number of queries. 
%When $m=O(n)$ holds in PC-based BlockCD algorithm, 
%the upper bound \memo{========} is expected to approximate the practical convergence rate, though it is not
%tight for $m=O(1)$ as pointed out in \memo{===============}. 
%\memo{==now computing: 2014-6-30(Mon): DONE 2014-6-29(Wed): result is not good for dim=300.==} 
%oracle_par <- list(kappa=2,del_zero=0.3,mu=0.01)
%On the other hand, the parallel implementation of BlockCD$[n,1]$ is not realized in the straightforward way. 

\begin{figure}[tp]
 %\hspace*{-3mm}
 \begin{center}
 \begin{tabular}{cc}
  \includegraphics[scale=0.4]{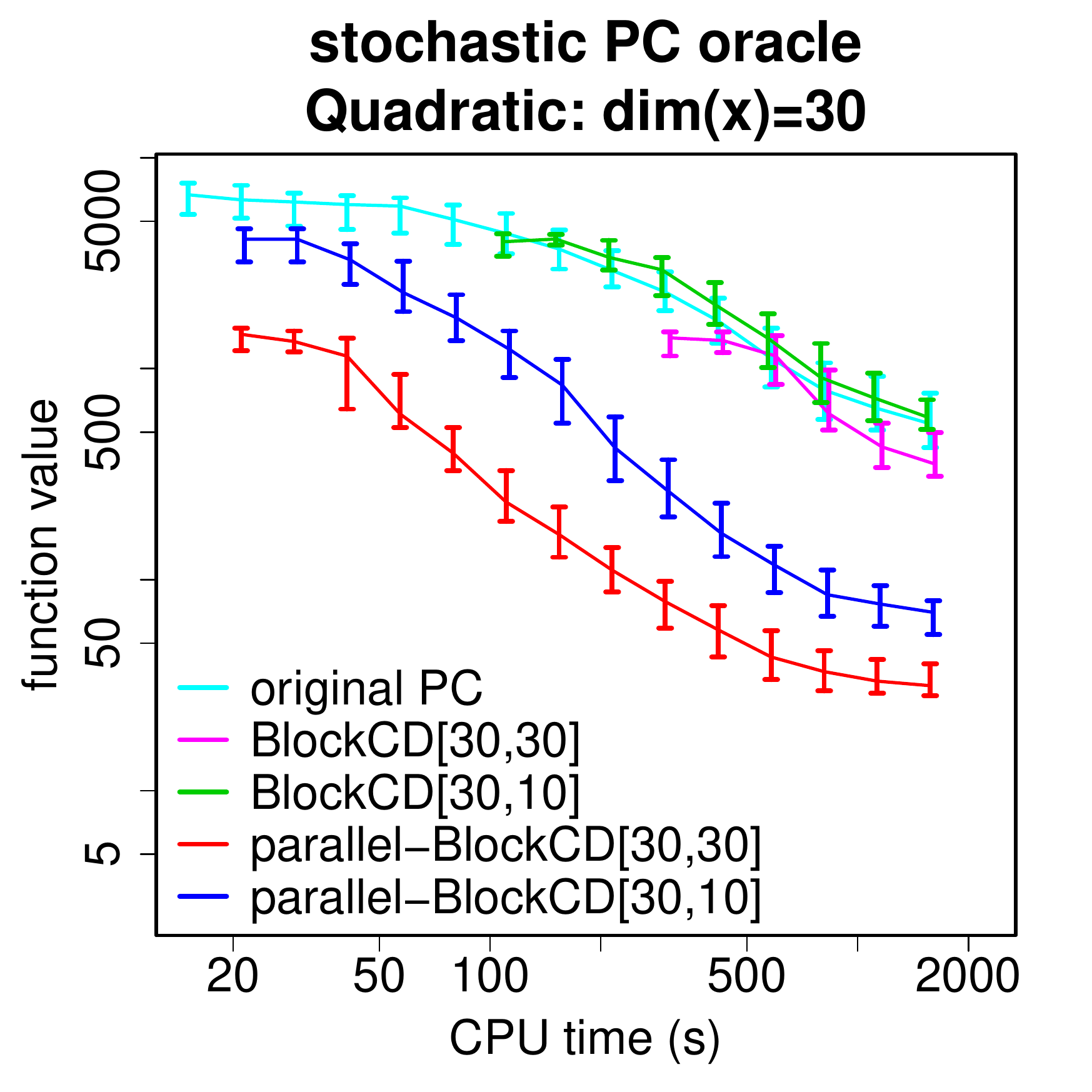}&
  \includegraphics[scale=0.4]{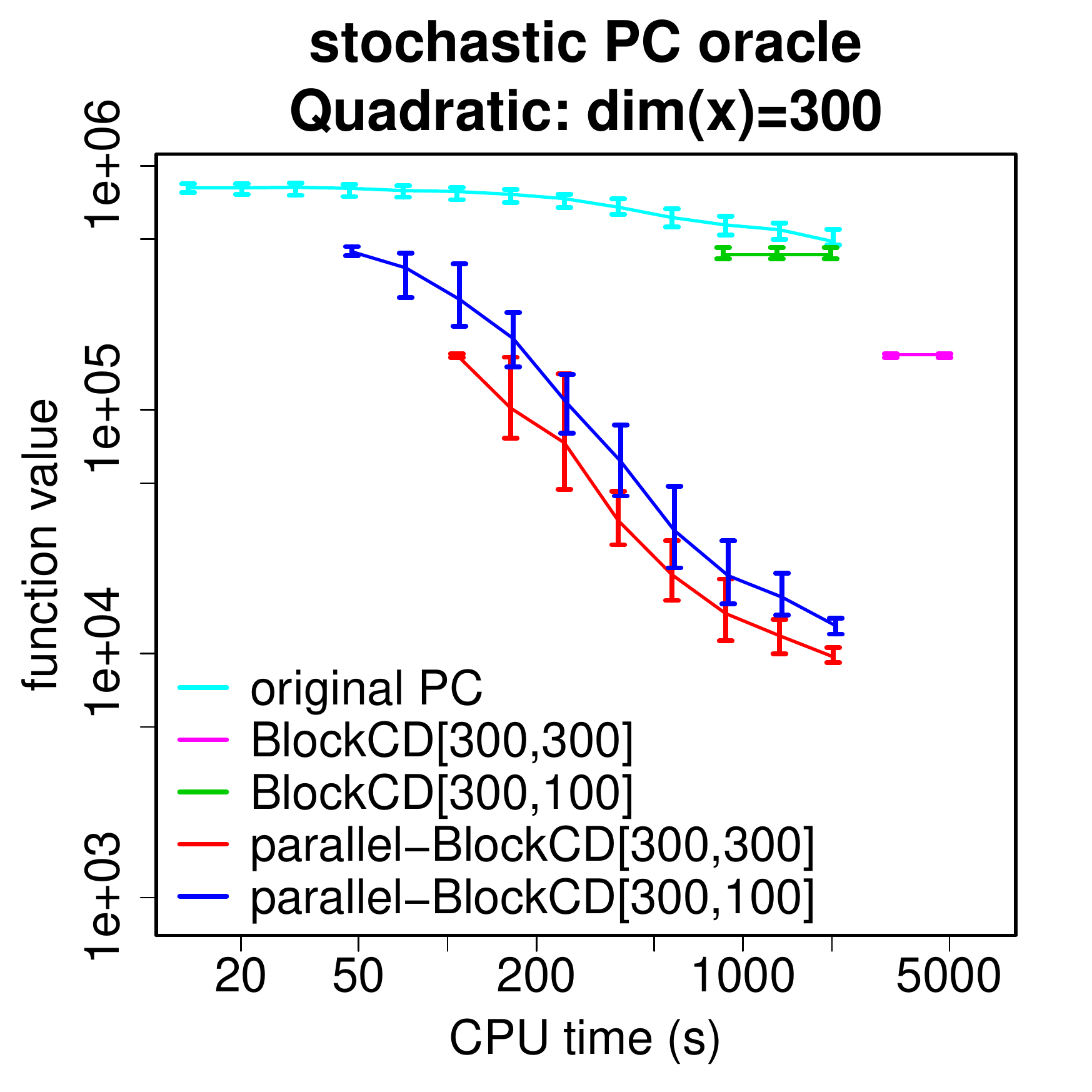}\\
  \includegraphics[scale=0.4]{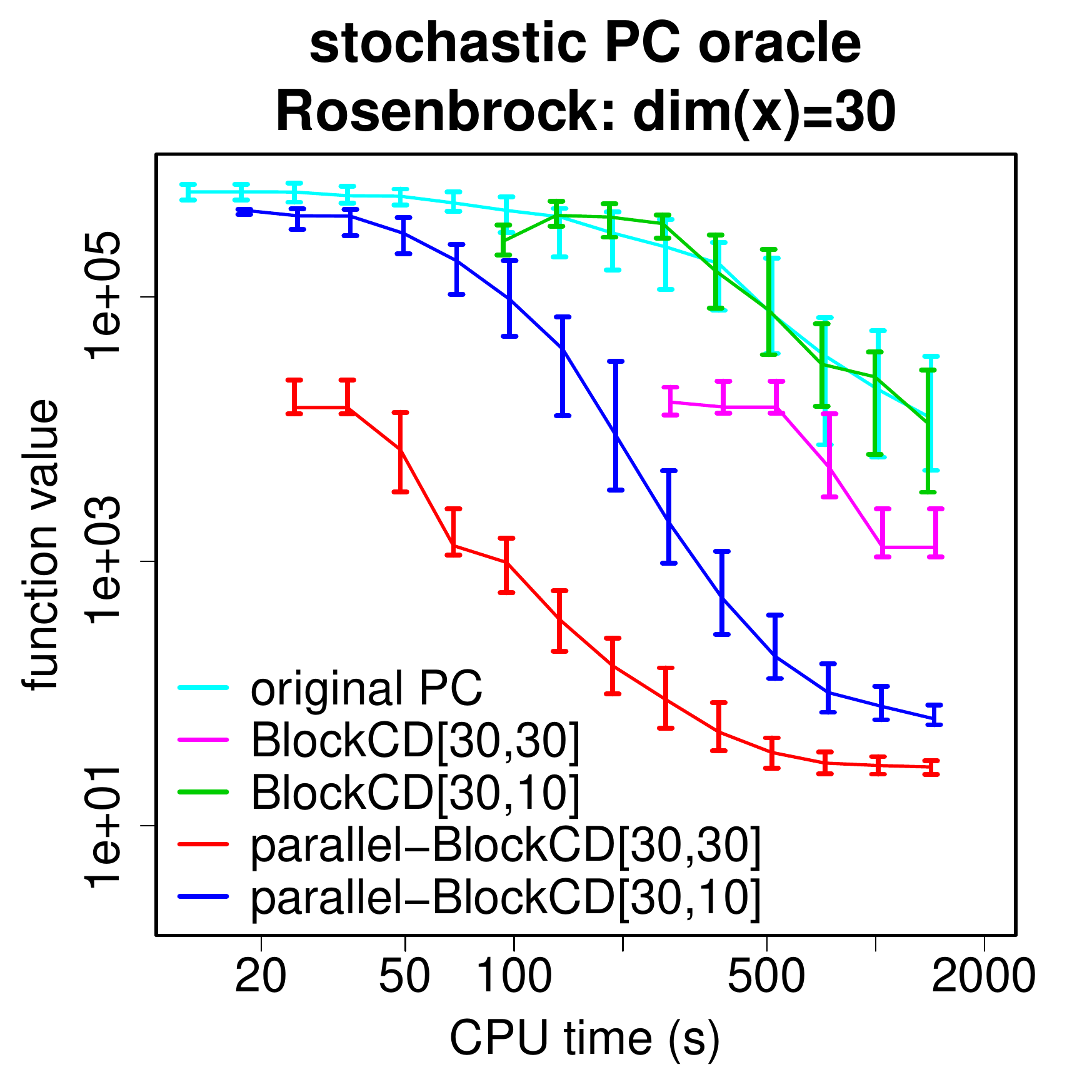}&
  \includegraphics[scale=0.4]{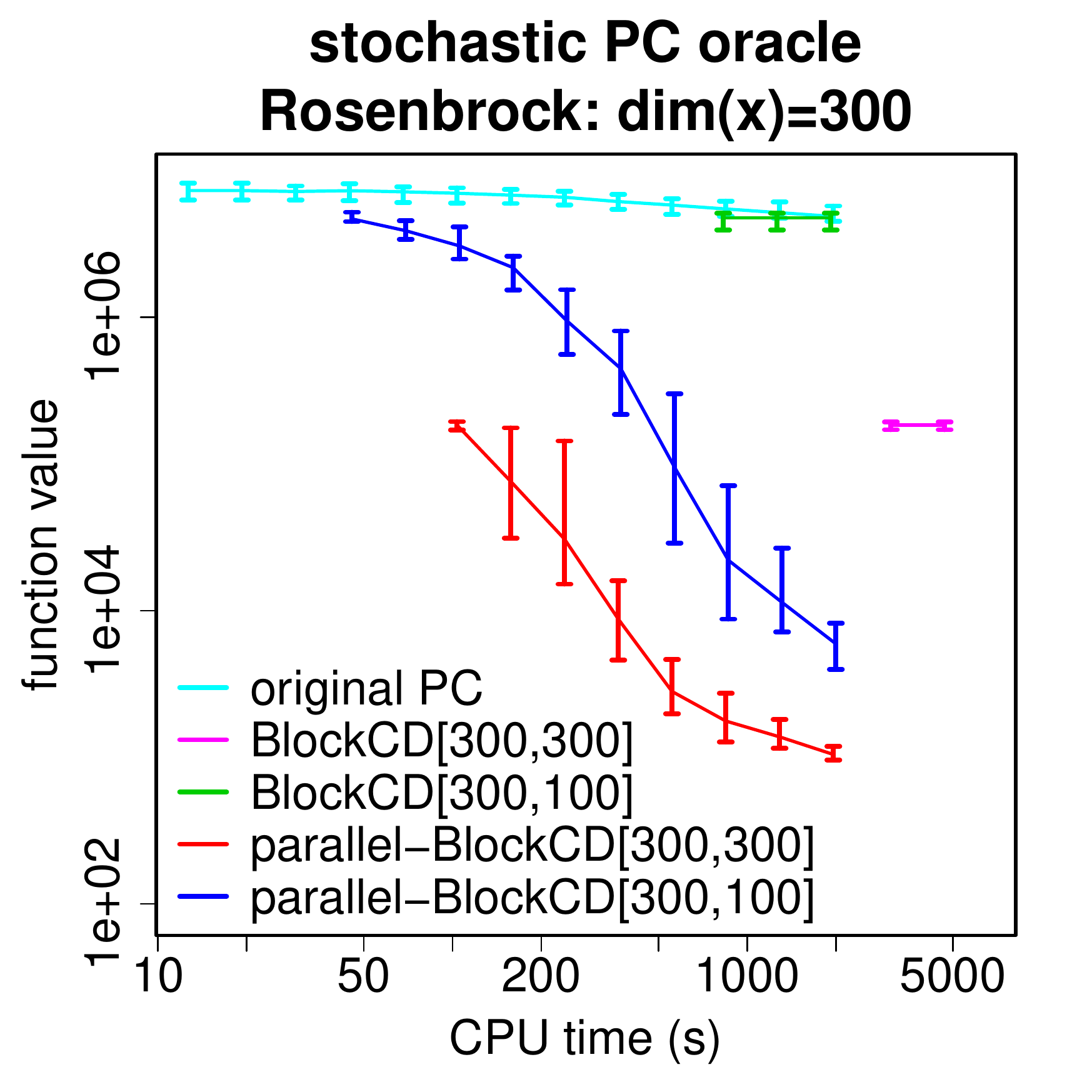}
 \end{tabular}
 \caption{
  Stochastic PC oracle is used in PC-based BlockCD algorithm. 
  Top panels: results in optimization of quadratic function. Bottom panels: results in
  optimization of Rosenbrock function. 
  The original PC algorithm,  BlockCD$[n,m]$ with $m=n$ and $m=n/3$, 
  and parallel-BlockCD$[n,m]$ with $m=n$ and $m=n/3$, 
  are compared for $n=30$ and $n=300$. 
  The median of the function value is shown to the CPU time (s). 
  The vertical bar shows the percentile $30\%$ to $70\%$. 
  }
 \label{fig:high_dim_problems_stochasticPC}
 \end{center}
\end{figure}

\section{Conclusion}
\label{con}
In this paper, we proposed a block coordinate descent algorithm for unconstrained optimization problems %by 
using the pairwise comparison of function values.
%Our algorithm estimates a modified Newton direction and performs a line search according
%to the estimated direction. 
Our algorithm consists of two steps: the direction estimate step and search step. 
The direction estimate step %in the algorithm 
can easily be parallelized. % and hence the
Hence, our algorithm is effectively applicable to large-scale optimization problems. 
%works in the large-scale optimization problems. 
%Theoretically, we provide an upper bound on the performance of the algorithm for a certain
%class of the objective functions and analyze  query complexity under both deterministic and stochastic  
%pairwise comparison oracle. 
Theoretically, we obtained an upper bound of the convergence rate and query complexity, 
when the deterministic and stochastic pairwise comparison oracles were used. 
%and analyze  query complexity under both deterministic and stochastic  
%pairwise comparison oracle. 
Practically, our algorithm is simple and easy to implement. %can easily be implemented. 
In addition, numerical experiments showed that the parallel implementation of our
algorithm outperformed the other methods. 
%the performance of our algorithm is good compared to the conventional methods.  
%Let us remark on some future works.
%Generalization of our algorithm to constrained optimization is an important subject. 
%Studying about the relation between pairwise comparison oracle and another type of oracle 
%information ({\it e.g.} gradient-sign oracle~\cite{ramdas2013algorithmic}) is an interesting problem. 
An extension of our algorithm to constrained optimization problems is an important future
work. Other interesting research directions include pursuing the relation between pairwise
comparison oracle and other kind of oracles such as gradient-sign
oracle~\cite{ramdas13:_algor_connec_activ_learn_stoch_convex_optim}.

\bibliographystyle{plain}
%\bibliography{paper}

%\bibliography{mybibfile}

\begin{thebibliography}{10}

\bibitem{audet2002analysis}
Charles Audet and John~E Dennis~Jr.
\newblock Analysis of generalized pattern searches.
\newblock {\em SIAM Journal on Optimization}, 13(3):889--903, 2002.

\bibitem{boyd2004convex}
Stephen~Poythress Boyd and Lieven Vandenberghe.
\newblock {\em Convex optimization}.
\newblock Cambridge university press, 2004.

\bibitem{conn2009introduction}
A~Andrew~R Conn, Katya Scheinberg, and Luis~N Vicente.
\newblock {\em Introduction to derivative-free optimization}, volume~8.
\newblock SIAM, 2009.

\bibitem{conn2000trust}
Andrew~R Conn, Nicholas~IM Gould, and Ph~L Toint.
\newblock {\em Trust region methods}, volume~1.
\newblock Siam, 2000.

\bibitem{Flaxman:2005:OCO:1070432.1070486}
A.~D. Flaxman, A.~T. Kalai, and H.~B. McMahan.
\newblock Online convex optimization in the bandit setting: Gradient descent
  without a gradient.
\newblock In {\em Proceedings of the Sixteenth Annual ACM-SIAM Symposium on
  Discrete Algorithms}, SODA '05, pages 385--394, Philadelphia, PA, USA, 2005.
  Society for Industrial and Applied Mathematics.

\bibitem{Fu06:sto_grad_est}
M.~C. Fu.
\newblock Gradient estimation.
\newblock In S.~G. Henderson and B.~L. Nelson, editors, {\em Handbooks in
  Operations Research and Management Science: Simulation}, chapter~19.
  Elservier Amsterdam, 2006.

\bibitem{gao12:_implem_nelder_mead_simpl_algor_adapt_param}
Fuchang Gao and Lixing Han.
\newblock Implementing the {N}elder-{M}ead simplex algorithm with adaptive
  parameters.
\newblock {\em Comput. Optim. Appl.}, 51(1):259--277, 2012.

\bibitem{jamieson12:_query_compl_deriv_free_optim}
K.~G. Jamieson, R.~D. Nowak, and B.~Recht.
\newblock Query complexity of derivative-free optimization.
\newblock In {\em NIPS}, pages 2681--2689, 2012.

\bibitem{kaariainen2006active}
Matti K{\"a}{\"a}ri{\"a}inen.
\newblock Active learning in the non-realizable case.
\newblock In {\em Algorithmic Learning Theory}, pages 63--77. Springer, 2006.

\bibitem{Lagarias98convergenceproperties}
Jeffrey~C. Lagarias, James~A. Reeds, Margaret~H. Wright, and Paul~E. Wright.
\newblock Convergence properties of the {N}elder-{M}ead simplex method in low
  dimensions.
\newblock {\em SIAM Journal of Optimization}, 9:112--147, 1998.

\bibitem{luenberger08:_linear_and_nonlin_progr}
D.~Luenberger and Y.~Ye.
\newblock {\em Linear and Nonlinear Programming}.
\newblock Springer, 2008.

\bibitem{mohri2012foundations}
Mehryar Mohri, Afshin Rostamizadeh, and Ameet Talwalkar.
\newblock {\em Foundations of machine learning}.
\newblock MIT press, 2012.

\bibitem{nelder65:_simpl_method_funct_minim}
J.~A. Nelder and R.~Mead.
\newblock A simplex method for function minimization.
\newblock {\em The Computer Journal}, 7(4):308--313, 1965.

\bibitem{team14:_r}
{R Core Team}.
\newblock {\em R: A Language and Environment for Statistical Computing}.
\newblock R Foundation for Statistical Computing, Vienna, Austria, 2014.

\bibitem{ramdas2013algorithmic}
Aaditya Ramdas and Aarti Singh.
\newblock Algorithmic connections between active learning and stochastic convex
  optimization.
\newblock In {\em Algorithmic Learning Theory}, pages 339--353. Springer, 2013.

\bibitem{ramdas13:_algor_connec_activ_learn_stoch_convex_optim}
Aaditya Ramdas and Aarti Singh.
\newblock Algorithmic connections between active learning and stochastic convex
  optimization.
\newblock In Sanjay Jain, R{\'e}mi Munos, Frank Stephan, and Thomas Zeugmann,
  editors, {\em ALT}, volume 8139 of {\em Lecture Notes in Computer Science},
  pages 339--353. Springer, 2013.

\bibitem{rios2013derivative}
Luis~Miguel Rios and Nikolaos~V Sahinidis.
\newblock Derivative-free optimization: A review of algorithms and comparison
  of software implementations.
\newblock {\em Journal of Global Optimization}, 56(3):1247--1293, 2013.

\end{thebibliography}

%\input{matsui14a.bbl}

%%%%%%%%%%%%%%%%%%%%%%%%%%%%%%%%%%%%%%%%%%%%%%%%%%%%%%%%%%%%
%%%%%%%%%%%%%%%%%%%%%%%%%%%%%%%%%%%%%%%%%%%%%%%%%%%%%%%%%%%%
%\newpage

\appendix

%\section*{Appendix}

%\section{Proof of Theorems}
%\label{pro}

%%%%%%%%%%%%%%%%%%%%%%%%%%%%%%%%%%%%%%%%%%%%%%%%%%%%%%%%%%%%
%\subsection{Lemmas} \label{proof.lem}

%Our algorithm recursively executes line search as a sub algorithm.
%The following lemma gives an upper bound of the number of queries for a line search algorithm and is essential to
%evaluate the total number of queries for our algorithm. 

%%%%%%%%%%%%%%%%%%%%%%%%%%%%%%%%%%%%%%%%%%%%%%%%%%%%%%%%%%%%
%\subsection{Proof of Theorem \ref{upper}} \label{proof.upper}
\section{Proof of Theorem \ref{upper}} \label{proof.upper}
\begin{proof}

The optimal solution of $f$ is denoted as $\x^*$. 
Let us define $\varepsilon'$ be $\varepsilon/(1+\frac{n}{m\gamma})$. 
If $f(\x_t)-f(\x^*)<\varepsilon'$ holds in the algorithm, we obtain
 $f(\x_{t+1})-f(\x^*)<\varepsilon'$, since the function value is non-increasing in each 
iteration of the algorithm\footnote{Monotone decrease of $f(\x_t)$ is assured by a minor
 modification of PC-oracle in~\cite{jamieson12:_query_compl_deriv_free_optim}.}. 

Next, we assume $\varepsilon'\leq{}f(\x_t)-f(\x^*)$. 
The assumption leads to
\begin{align*}
2\sigma\varepsilon'\leq2\sigma(f(\x_t)-f(\x^*))\leq\|\nabla{f}(\x_t)\|^2, 
\end{align*}
in which the second inequality is derived from (9.9) in \cite{boyd2004convex}. 
In the following, we use the inequality 
\begin{align*}
 f(\x_t+\beta_t\d_t/\|\d_t\|)\leq{} f(\x_t)-\frac{|\nabla{f}(\x_t)^T\d_t|^2}{2L\|\d_t\|^2}+\frac{L}{2}\eta^2
\end{align*}
that is proved in \cite{jamieson12:_query_compl_deriv_free_optim}. 
For the $i$-th coordinate, let us define the functions $g_{\mathrm{low}}(\alpha)$ and
 $g_{\mathrm{up}}(\alpha)$ as  
\begin{align*}
g_{\mathrm{low}}(\alpha)
 =f(\x_t)+\frac{\partial{f}(\x_t)}{\partial{x}_i}\alpha+\frac{\sigma}{2}\alpha^2,  \quad\text{and}\quad
g_{\mathrm{up}}(\alpha)
 =f(\x_t)+\frac{\partial{f}(\x_t)}{\partial{x}_i}\alpha+\frac{L}{2}\alpha^2. 
\end{align*}
Then, we have
\begin{align*}
g_{\mathrm{low}}(\alpha)\leq{}f(\x_t+\alpha{\e_i}) \leq{} g_{\mathrm{up}}(\alpha). 
\end{align*}
Let $\alpha_{\mathrm{up}}$ and $\alpha_i^*$ be the minimum solution of
 $\min_{\alpha}g_{\mathrm{up}}(\alpha)$ and $\min_{\alpha}f(\x_t+\alpha{\e_i})$,
 respectively. Then, we obtain
\begin{align*}
 g_{\mathrm{low}}(\alpha_i^*)
 \leq
 f(\x_t+\alpha_i^*{\e_i})
 \leq
 f(\x_t+\alpha_{\mathrm{up}}{\e_i})
\leq
 g_{\mathrm{up}}(\alpha_{\mathrm{up}}). 
\end{align*}
The inequality 
$g_{\mathrm{low}}(\alpha_i^*)\leq{} g_{\mathrm{up}}(\alpha_{\mathrm{up}})$ yields that 
$\alpha_i^*$ lies between 
$-c_0\frac{\partial{f}(x_t)}{\partial{x_i}}$ and 
$-c_1\frac{\partial{f}(x_t)}{\partial{x_i}}$, where
$c_0$ and $c_1$ are defined as 
\begin{align*}
 c_0=(1-\sqrt{1-\sigma/L})/\sigma,\quad
 c_1=(1+\sqrt{1-\sigma/L})/\sigma. 
\end{align*}
Here, $0<c_0\leq{}c_1$ holds. Each component of the search direction
$\d_t=(d_1,\ldots,d_n)\neq\0$ in Algorithm~\ref{alg1}  
satisfies $|d_i-\alpha_i^*|\leq\eta$ if $i=i_k$ and otherwise $d_i=0$. 
For $I=\{i_1,\ldots,i_m\}\subset\{1,\ldots,n\}$, let $\|\a\|_{I}^2$ of the vector
$\a\in\Rbb^n$ be $\sum_{i\in{I}}a_{i}^2$. 
%The relation between $d_i$ and $\alpha_i^*$ and 
Then, the triangle inequality leads to 
\begin{align*}
 \|\d_t\|& \leq  c_1\|\nabla{f}(\x_t)\|_{I} + \sqrt{m}\eta,\\ 
|\nabla{f}(\x_t)^T\d_t|
 &\geq
 c_0\|\nabla{f}(\x_t)\|_{I}^2-\sqrt{m}\eta \|\nabla{f}(\x_t)\|_{I}. 
\end{align*}
 The assumption $\varepsilon'\leq{}f(\x_t)-f(\x^*)$ and the inequalities 
 $2\sigma(f(\x_t)-f(\x^*))\leq\|f(\x_t)\|^2,\,1/4L^2\leq{c_0}^2$
 lead to 
\begin{align*}
 \eta
 =
% \sqrt{\frac{\varepsilon'}{L}}
 \sqrt{\frac{\varepsilon'\sigma}{8L^2n}}
 \leq
 c_0\sqrt{\frac{\sigma\varepsilon'}{2n}}\leq{}c_0\frac{\|\nabla{f}(\x_t)\|}{2\sqrt{n}}. 
\end{align*}
Hence, we obtain
\begin{align*}
 \|\d_t\|& \leq  c_1\|\nabla{f}(\x_t)\|_{I} + \frac{c_0}{2}\sqrt{\frac{m}{n}}\|\nabla{f(\x_t)}\|,\\ 
 |\nabla{f}(\x_t)^T\d_t|
 &\geq
 \left[
 c_0\|\nabla{f}(\x_t)\|_{I}^2-\frac{c_0}{2}\sqrt{\frac{m}{n}}\|\nabla{f}(\x_t)\|\|\nabla{f}(\x_t)\|_{I}
 \right]_+, 
\end{align*}
where $[x]_+=\max\{0,x\}$ for $x\in\Rbb$. 
Let $Z=\sqrt{\frac{n}{m}}\|\nabla{f}(x_t)\|_I/\|\nabla{f}(x_t)\|$ be a non-negative valued
 random variable defined from the random set $I$, and define the non-negative value $k$ as
 $k=c_0/c_1\leq{1}$. 
A lower bound of the expectation of $(|\nabla{f}(\x_t)^T\d_t|/\|\d_t\|)^2$ with respect to the
distribution of $I$ is given as 
\begin{align*}
 \Ebb_I\left[ \left(\frac{|\nabla{f}(\x_t)^T\d_t|}{\|\d_t\|}\right)^2 \right]
& \geq
 \Ebb_I\left[
 \left(
 \frac{ \left[
 c_0\|\nabla{f}(\x)\|_{I}^2-\frac{c_0}{2}\sqrt{\frac{m}{n}}\|\nabla{f}(\x_t)\|\|\nabla{f}(\x_t)\|_I
 \right]_+}
 {c_1\|\nabla{f}(\x_t)\|_I + \frac{c_0}{2}\sqrt{\frac{m}{n}}\|\nabla{f(\x_t)}\|}
 \right)^2
 \right] \\
& =
 k^2\frac{m}{n}\|\nabla{f(\x_t)}\|^2\Ebb_I\left[{Z^2} \frac{[Z-1/2]_+^2}{(Z+k/2)^2}\right]\\
& \geq
 k^2\frac{m}{n}\|\nabla{f(\x_t)}\|^2 \Ebb_I\left[{Z^2}
 \frac{[Z-1/2]_+^2}{(Z+1/2)^2}\right]. 
\end{align*}
The random variable $Z$ is non-negative, and $\Ebb_I[Z^2]=1$ holds. 
Thus, Lemma~\ref{eqn:lemma_bound_expZ} in the below leads to 
\begin{align*}
 \Ebb_I\left[ \left(\frac{|\nabla{f}(\x_t)^T\d_t|}{\|\d_t\|}\right)^2 \right]
 &\geq
 \frac{k^2}{53}\frac{m}{n}\|\nabla{f}(\x_t)\|^2. 
\end{align*}
Eventually, if $\varepsilon'\leq{}f(\x_t)-f(\x^*)$, the conditional expectation of 
$f(\x_{t+1})-f(\x^*)$ for given $\d_0,\d_1,\ldots,\d_{t-1}$ is given as 
\begin{align*}
 \Ebb[f(\x_{t+1})-f(\x^*)|\d_0,\ldots,\d_{t-1}]
& \leq
 f(\x_t)-f(\x^*)- \frac{k^2}{106L}\frac{m}{n}\|\nabla{f}(\x_t)\|^2 +\frac{L\eta^2}{2}\\
& \leq
 \left(1-\frac{m}{n}\gamma\right)(f(\x_t)-f(\x^*)) +\frac{L\eta^2}{2}. 
\end{align*}
Combining the above inequality with the case of $f(\x_t)-f(\x^*)<\varepsilon'$, we obtain
\begin{align*}
&\phantom{\leq} 
 \Ebb[f(\x_{t+1})-f(\x^*)|\d_{0},\ldots,\d_{t-1}]\\
 & \leq
 \1[f(\x_{t})-f(\x^*)\geq\varepsilon']\cdot \left[ 
 \left(1-\frac{m}{n}\gamma\right)(f(\x_{t})-f(\x^*)) +\frac{L\eta^2}{2}\right]
 + \1[f(\x_{t})-f(\x^*)<\varepsilon'] \cdot\varepsilon'. 
\end{align*}
The expectation with respect to all $\d_0,\ldots,\d_t$ yields 
\begin{align*}
 \Ebb[f(\x_{t+1})-f(\x^*)]
& \leq
 \left(1-\frac{m}{n}\gamma\right)  \Ebb[\1[f(\x_{t})-f(\x^*)\geq\varepsilon'](f(\x_{t})-f(\x^*))] \\
& \phantom{\leq}
 +\Ebb[\1[f(\x_{t})-f(\x^*)\geq\varepsilon']]\frac{L\eta^2}{2}
 + \Ebb[\1[f(\x_{t})-f(\x^*)<\varepsilon'] ]\varepsilon' \\
& \leq
 \left(1-\frac{m}{n}\gamma\right) \Ebb[f(\x_{t})-f(\x^*)]
 +\max\left\{ \frac{L\eta^2}{2},\,\varepsilon' \right\}.
\end{align*}
Since $0<\gamma<1$ and $\max\{L\eta^2/2,\,\varepsilon'\}=\varepsilon'$ hold, 
for $\Delta_T=\Ebb[f(\x_T)-f(\x^*)]$ we have 
\begin{align*}
 \Delta_{T} -  \frac{n}{m}  \frac{\varepsilon'}{\gamma} 
\leq \left(1 -  \frac{m}{n}\gamma\right) \left(\Delta_{T-1} - \frac{n}{m}  \frac{\varepsilon'}{\gamma} 
 \right) 
\leq 
\left(1 -  \frac{m}{n}\gamma\right)^T \Delta_0. 
\end{align*}
When $T$ is greater than $T_0$ in \eqref{eqn:def_T0}, we obtain
$\left(1 - \frac{m}{n}\gamma\right)^{T} \Delta_{0}\leq\varepsilon'$
and 
\begin{align*}
% \Ebb[f(\x_T)-f(\x^*)] 
\Delta_T\leq \varepsilon'\left(1+\frac{n}{m\gamma}\right)
 =
 \varepsilon. 
\end{align*}
Let us consider the accuracy of the numerical solution $\x_T$. 
As shown in \cite[Chap.~9]{boyd2004convex}, the inequality
\begin{align*}
 \|\x-\x^*\|^2\leq\frac{8L}{\sigma^2}(f(\x)-f(\x^*))
\end{align*}
holds. Thus, for $T\geq{T_0}$, we have 
\begin{align*}
 \mathbb{E}[\|\x_T-\x^*\|]^2\leq
 \mathbb{E}[\|\x_T-\x^*\|^2]\leq\frac{8L}{\sigma^2}\varepsilon
 =
 64n\left(\frac{L}{\sigma}\right)^3\left(1+\frac{n}{m\gamma}\right)\eta^2. 
\end{align*}
\end{proof}

%%%%%%%%%%%%%%%%%%%%%%%%%%%%%%%%%%%%%%%%%%%%%%%%%%%%%%%%%%%%
\begin{lem}
 \label{eqn:lemma_bound_expZ}
 Let $Z$ be a non-negative random variable satisfying $\Ebb[Z^2]=1$. 
 Then, we have
 \begin{align*}
  \Ebb\left[{Z^2} \frac{[Z-1/2]_+^2}{(Z+1/2)^2}\right]\geq\frac{1}{53}.
 \end{align*}
\end{lem}
\begin{proof}
%Note that $\Ebb_I[Z^2]=1$ and $0\leq{Z}$ hold. 
For $z\geq0$ and $\delta\geq0$, we have the inequality
\begin{align*}
 \frac{[z-1/2]_+^2}{(z+1/2)^2}
 \geq
 \frac{\delta^2}{(1+\delta)^2}\1[z\geq1/2+\delta]. 
\end{align*}
%For the non-negative random variable $Z$ with $\Ebb_I[Z^2]=1$, we get 
Then, we get
\begin{align*}
 \Ebb\left[{Z^2} \frac{[Z-1/2]_+^2}{(Z+1/2)^2}\right]
 & \geq
 \frac{\delta^2}{(1+\delta)^2}\Ebb[Z^2 \1[Z\geq1/2+\delta]]\\
 &=
 \frac{\delta^2}{(1+\delta)^2}\Ebb[Z^2(1-\1[Z<1/2+\delta])]\\
 &= 
 \frac{\delta^2}{(1+\delta)^2}
 \left( 1-\Ebb[Z^2\1[Z<1/2+\delta]] \right)\\
 &\geq
 \frac{\delta^2}{(1+\delta)^2}
 \left(
 1-(1/2+\delta)^2\Pr(Z<1/2+\delta)
 \right)\\
&\geq
 \frac{\delta^2}{(1+\delta)^2} \left( 1-(1/2+\delta)^2 \right). 
\end{align*}
%By setting $\delta=477-$, we have 
By setting $\delta$ appropriately, we obtain
\begin{align*}
 \Ebb_I\left[{Z^2} \frac{[Z-1/2]_+^2}{(Z+1/2)^2}\right]\geq\frac{1}{53}.
\end{align*}
%\begin{flushright} $ \Box$ \end{flushright}
\end{proof}

\section{Proof of Corollary \ref{convergence}}  \label{proof.convergence}
\begin{proof}
For the output $\hat{\x}_Q$ of BlockCD$[n,m]$,
$f(\hat{\x}_Q)\ge f(\hat{\x}_{Q+1})$ holds, and thus, the sequence $\{\hat{\x}_Q\}_{Q\in\N}$ is included in
\begin{eqnarray*}
C(x_0):=\{\x\in\R^n|f(\x)\le f(\x_0)\}.
\end{eqnarray*}
Since $f$ is convex and continuous, $C(\x_0)$ is convex and closed.
Moreover, since $f$ is convex and it has non-degenerate Hessian,
the Hessian is positive definite, and thus, $f$ is strictly convex. 
Then $C(\x_0)$ is bounded as follows.
We set 
%the minimum value of $f$ and 
the minimul directional derivative along the radial direction from $\x^*$ over the unit sphere around $\x^*$ as
\begin{eqnarray*}
%a&:=&\min_{||u||=1} f(x_f^*+u)\\
b&:=&\min_{\|{\bm u}\|=1} \nabla f(\x^*+\u)\cdot \u.
\end{eqnarray*}
Then, $b$ is strictly positive and the following holds for any $\x\in C(x_0)$ such that $\|\x-\x^*\| \ge 1$, 
\begin{eqnarray*}
b\|\x-\x^*\|+(f(\x^*)-b)
\le f(\x)
\le f(\x_0).
\end{eqnarray*}
Thus we have 
\begin{eqnarray}
C(\x_0) \subset \left\{\x\bigg| \|\x-\x^*\| \le 1+\frac{f(\x_0)-f(\x^*)}{b}\right\}.
\label{include}
\end{eqnarray}
Since the right hand side of (\ref{include}) is a bounded ball, $C(\x_0)$ is also bounded.
Thus, $C(\x_0)$ is a convex compact set.

Since $f$ is twice continuously differentiable,
the Hessian matrix $\nabla^2f(\x)$ is continuous with respect to $\x\in\R^n$.
By the positive definiteness of the Hessian matrix,
the minimum and maximum eigenvalues $e_{min}(\x)$ and $e_{max}(\x)$ of $\nabla^2f(\x)$ are continuous and positive.
%see http://math.stackexchange.com/questions/480619/eigenvalues-of-matrix-with-entries-that-are-continuous-functions
Therefore, there are the positive minimum value $\sigma$ of $e_{min}(\x)$ and maximum value $L$ of $e_{max}(\x)$ on the compact set $C(\x_0)$.
It means that $f$ is $\sigma$-strongly convex and $L$-Lipschitz on $C(\x_0)$.
Thus, the same argument to obtain (\ref{eq:quecom}) can be applied for $f$.
%the difference $f(\hat{\x}_Q)-f(\x^*)$ sub-exponentially goes to less than $\epsilon$ with respect to 
%the number $Q$ of queries by Theorem~\ref{upper}. 
\end{proof}

\end{document}